\documentclass{article} % For LaTeX2e
\usepackage{iclr2027_conference,times}

\usepackage{amsmath,amsfonts,bm}

\def\secref#1{Section~\ref{#1}}
\def\eqref#1{Eq.~\ref{#1}}
\def\Eqref#1{Equation~\ref{#1}}
\def\1{\bm{1}}

\DeclareMathAlphabet{\mathsfit}{\encodingdefault}{\sfdefault}{m}{sl}
\SetMathAlphabet{\mathsfit}{bold}{\encodingdefault}{\sfdefault}{bx}{n}

\usepackage{url}
\usepackage{xcolor}
\usepackage{graphicx}
\usepackage{enumitem}
\usepackage{amsmath}
\usepackage{amssymb}
\usepackage{wrapfig}
\usepackage{mathtools}
\usepackage{amsthm}
\usepackage{hyperref}
\usepackage{url}
\usepackage{booktabs}
\usepackage{subcaption}
\usepackage{multirow}
\usepackage{bbm}
\usepackage[table]{xcolor}
\definecolor{highlight}{RGB}{235, 240, 255}
\usepackage{mathrsfs}
\usepackage{listings}
\usepackage{titletoc}
\usepackage{algorithm}
\usepackage{algpseudocode}

\theoremstyle{plain}
\newtheorem{theorem}{Theorem}[section]
\newtheorem{proposition}[theorem]{Proposition}
\newtheorem{lemma}[theorem]{Lemma}

\theoremstyle{definition}

\newtheorem{assumption}[theorem]{Assumption}
\theoremstyle{remark}
\newtheorem{remark}[theorem]{Remark}

\title{JOVE: Joint Execution and Verification for Resource-Aware LLM Task Graphs}
\author{Haoran Zhang$^1$, Dongjun Kim$^{1}$, Seohyeon Cha$^{1}$, Kevin S Chan$^{2}$, Ananthram Swami$^{2}$\\ 
\textbf{Gustavo De Veciana$^{1}$, Haris Vikalo$^{1}$} \\
$^{1}$University of Texas at Austin \quad 
$^{2}$DEVCOM Army Research Laboratory
}
\iclrfinalcopy

\begin{document}

\maketitle
\lhead{Preprint}

% Sep 21 comments
% Consider replace "verification"
% highlight that we don't know any final response correctness, and our feedback is paid, unlike common RL/multi-arm bandits that action automatically receives feedback. (done)
% provide graph and latency stationary or iid assumptions
% how latency is set
% mu_t, delta_t in assumptions to be positive
% reference model, how we set the budget and latency constraints: choose a aggresive/simple dataset
\vspace{-0.5cm}
\begin{abstract}
Complex reasoning queries can be decomposed into directed acyclic task graphs and distributed across heterogeneous LLMs, reducing latency through parallelism and enabling smaller models to solve complex tasks.
In practice, however, the suitability of an LLM for a given subtask may be a priori unknown, and execution alone does not reveal output correctness.
We propose JOVE, an online framework that jointly assigns executor LLMs and selects intermediate outputs for paid verification.
Verification runs asynchronously and is used to improve future allocations, so the system must balance spending on execution now against learning for later.
We study how to optimize this trade-off under a long-term budget and a per-query latency constraint, with stochastic, initially unknown LLM service quality, invocation costs, and execution times. JOVE makes execution and verification decisions by solving a sequence of per-query mixed-integer linear programs. Online learning updates task-dependent estimates of LLM quality based on verification feedback, while an information-gain bonus incorporates the value of learning into allocation decisions. 
Under a natural set of assumptions, we establish sublinear quality-learning regret for JOVE. 
Across four reasoning benchmarks, JOVE achieves competitive accuracy against standard inference baselines while reducing average cost and latency by at least 3.17 times.
Code: \href{https://anonymous.4open.science/r/JOVE/README.md}{link}. 
\end{abstract}

\section{Introduction}
Large language models (LLMs) offer different combinations of capability, specialization, and serving cost.
This heterogeneity has motivated \emph{LLM routing}: instead of sending every query to the strongest model, a router might select a less expensive model when that is likely to suffice~\citep{chen2023frugalgpt,ong2024routellm,wang2025mixllm}.
Model selection is of particular interest for complex reasoning queries, which often benefit from an agentic workflow that decomposes the original problem into smaller, interdependent subtasks. A natural abstraction is \emph{graph-based inference}: a planner decomposes a query into a directed acyclic graph (DAG) whose nodes represent subtasks and whose edges indicate dependencies between them. The node outputs are then used to construct the final response. 
This structure enables lightweight models to handle individual subtasks and reduces latency through parallel execution of independent branches.
% while allowing each subtask to be assigned to an LLM based on difficulty and cost. 
% In this paper, we study how to allocate LLM API calls within such task graphs under cost and latency constraints.
% It is suitable for edge devices, where a lightweight local planner can orchestrate heterogeneous remote APIs without hosting multiple LLMs on-device.

Prior work spans three directions. (1) Graph-structured decoding organizes intermediate reasoning or exposes parallelism, typically with a fixed LLM~\citep{ning2024sot,jin2025plato}.
(2) LLM routing selects models per query or through sequential interactions~\citep{chen2023frugalgpt,zhang2025routerr1}. 
(3) Graph-aware LLM systems coordinate different model calls within task graphs~\citep{kim2024llmcompiler,dong2026hybridflow}.
Existing routing and graph-aware allocation methods typically learn from offline data or assume free model-quality feedback after each call. However, in practice, obtaining such feedback may require a judge model or human evaluator at an extra cost.
\textbf{In this paper, we therefore ask:} given a query and its generated task graph, how should a resource-constrained system jointly select executor LLMs and decide which node outputs to verify, balancing current-query performance against paid verification that improves future allocations?

% \footnote{Our framework can also incorporate non-LLM tools; we discuss this extension in Appendix~\ref{app}.}

\begin{figure}
    \centering
    \includegraphics[width=0.87\linewidth]{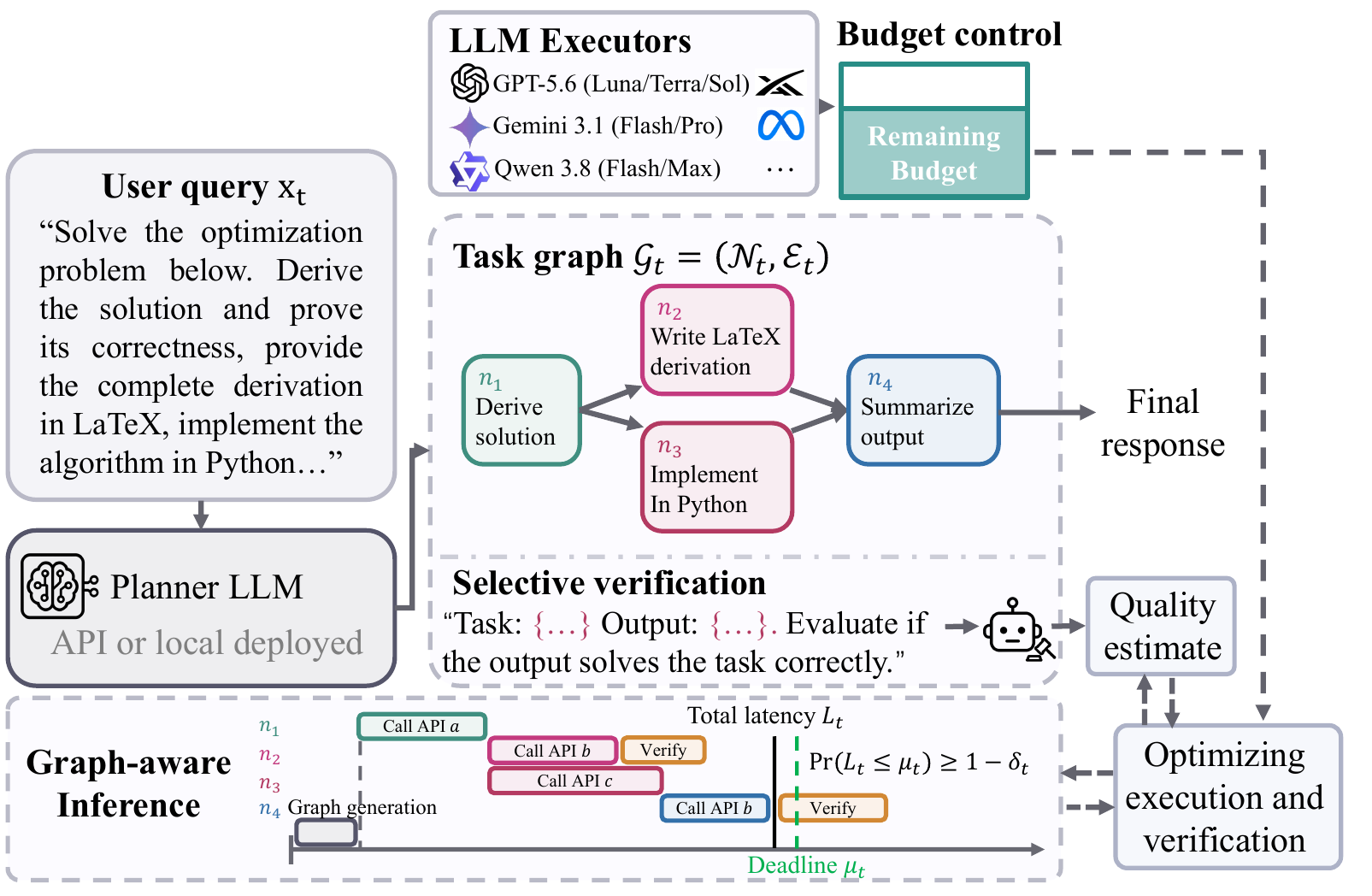}
    \caption{Overview of JOVE. A planner decomposes each query into a task graph and identifies candidate executor LLMs for each node. JOVE jointly assigns executors and selects outputs for verification, subject to a long-term execution/verification budget and latency constraints. Verification proceeds asynchronously without delaying the current response, improving model-quality estimates.}
    \label{fig:overview}
\vspace{-0.2cm}
\end{figure}

\textbf{Challenges and our design.}
We study the problem under two practical resource constraints: a long-term average \textit{API-cost budget} shared across queries, which offers the flexibility to absorb cost fluctuations over rounds, and \textit{delivery deadlines} set per query according to its urgency and difficulty.
(1) The first challenge is allocating heterogeneous LLMs under uncertain model quality, cost, and latency while satisfying both constraints. 
The budget limit couples decisions across queries: spending more on difficult queries requires saving on others. 
The deadline couples decisions within each task graph: choosing a slower executor at one node can restrict feasible choices downstream, even when it is locally favorable in quality and cost.
(2) The second challenge is obtaining sufficiently granular model-quality feedback without wasting resources by verifying every output.
A model call reveals its cost and latency but not output correctness, and the final response feedback may be too coarse to attribute errors to individual node-model assignments.
We therefore introduce \textbf{selective node-level verification} to provide feedback on individual outputs at an extra cost.
Each node contains a narrowly scoped subtask and an explicit success criterion. 
After a selected node completes, its subtask, outputs, and criterion are sent to a fixed verifier (a human evaluator or a judge model), which returns a binary correctness label. This label updates the estimated quality of the selected model for similar subtasks, informing future allocations. Verification is not required to construct the current response and can proceed asynchronously after node completion. We therefore assume that it does not delay response delivery, although its cost counts toward the long-term query cost. 
% budget, but its latency is excluded from the current query's response time. 

In this paper, we propose \textbf{JOVE}, \textbf{J}oint \textbf{O}ptimization of \textbf{V}erification and \textbf{E}xecution, a framework that jointly selects executor LLMs
% \footnote{An extension to include non-LLM executors is discussed in Appendix \ref{app:nonllm}.} 
and outputs to verify under system resource constraints (Figure~\ref{fig:overview}). We formulate the problem as an online constrained optimization problem that maximizes average node correctness subject to a long-term average API-cost budget and a probabilistic latency limit for each query. JOVE makes allocation decisions by solving per-query mixed-integer linear programs (MILPs), using online estimates of model quality, cost, and latency, with an information-gain bonus that values verification feedback for future allocations.
% We further characterize its exploration behavior and establish a sublinear regret guarantee under the stated assumptions.
Our main contributions are:
\begin{itemize}[wide]
    \item \textbf{Joint execution--verification formulation.}
    % We formulate the joint selection of executor LLMs and outputs for verification in task-graph inference with initially unknown model quality, cost, and latency, under a long-term API-cost budget and per-query latency chance constraint. 
    We formulate task-graph inference as an online problem of jointly choosing executor LLMs and which outputs to verify, with initially unknown model quality, cost, and latency, under a long-term cost budget and a per-query latency chance constraint.
    \item \textbf{Verification-aware online allocation.}
    % We develop JOVE, a per-query MILP policy that uses online quality estimates and an information-gain bonus to jointly select executor models and outputs to verify. Under stated assumptions, we establish sublinear quality-learning regret when the verification weight lies within a specified interval.
    We develop JOVE, a per-query MILP policy combining online quality estimates with an information-gain bonus for verification. We prove sublinear quality-learning regret when verification is frequent enough to learn but does not crowd out execution.
    \item \textbf{Empirical evaluation.}
    % We evaluate JOVE on multiple reasoning benchmarks in terms of answer quality, total API cost, and deadline compliance. Ablations examine resource constraints, planner models, and verification strategies, including sensitivity to verification weights.
    Across multiple reasoning datasets, JOVE achieves average accuracy comparable to the strongest baseline while achieving 3.7--16.3$\times$ lower cost and 8.3--17.0$\times$ lower latency. Ablation studies examine the effects of different verification weights and resource constraints.
\end{itemize}
% \begin{itemize}[wide]
% \item \textbf{Resource-constrained graph-based inference.}
% We formulate heterogeneous API allocation over a planner-generated task graph under a long-term API budget and a per-query latency chance constraint. Unlike scalarized routing objectives, our formulation enforces the two operational requirements separately.
% \item \textbf{Joint execution and verification.} We introduce a verification-aware online policy that jointly assigns executor APIs and selects intermediate outputs for verification. The resulting node-level feedback supports online quality estimation without requiring user-provided final-response labels.
% \item \textbf{Tractable constrained API optimization and online learning.} We reduce the time-average budget to an adaptive dual price and the latency chance constraint to deterministic critical-path inequalities, yielding a per-query MILP. We also characterize the exploration induced by verification and establish sublinear regret under the stated assumptions.
% \item \textbf{Empirical evaluation.}
% Experiments across multiple reasoning benchmarks demonstrate improved quality--cost--latency trade-offs over both inference-level and graph-aware routing baselines.
% \end{itemize}

\section{Related Work}
\label{sec:related}

\textbf{Graph-structured inference.}
Graph structure can improve reasoning or expose opportunities for parallel execution. Tree-of-Thoughts and Graph-of-Thoughts organize and search over intermediate reasoning states~\citep{yao2023tot,besta2024got}. Skeleton-of-Thought, PASTA, and Plato instead exploit independence or explicit dependencies to generate parts of a response concurrently~\citep{ning2024sot,jin2025learning,jin2025plato}. HuggingGPT, LLMCompiler, and S-DAG extend workflow planning to collections of different models, tools, and subagents~\citep{shen2023hugginggpt,kim2024llmcompiler,dong2026s}. Dyserve further includes verifiers to gate model selection using static offline profiles~\citep{qian2026servingagenticworkflowsphysicalplan}.
% These approaches study how to structure reasoning and execution. 
JOVE takes a planner-generated DAG as input and learns online to optimize executor assignments and which outputs to verify under resource constraints. 
% It can therefore complement methods for constructing task graphs.

\textbf{Resource-aware LLM routing and paid feedback.}
LLM routers and cascades improve quality-cost trade-offs by selecting models based on predicted quality, difficulty, or preference~\citep{chen2023frugalgpt,ding2024hybridllm,ong2024routellm,dekoninck2025unified,feng2024graphrouter}.
Recent online and resource-aware methods further consider feedback, budgets, or latency in their problems~\citep{wang2025mixllm,panda2025pilot,li2026wiserouter,patel2026latencyaware,zhang2025routerr1,atalar2025neural}, mostly assuming that quality feedback is freely available after execution.
AutoMix verifies an initial output before routing, incurring verification cost and delay on every query rather than verifying selectively~\citep{aggarwal2024automix}.
HybridFlow extends routing to task graphs, training its policy on offline data with subsequent online calibration, but restricts each routing decision to either an edge or a cloud model~\citep{dong2026hybridflow}.
JOVE instead selectively verifies node outputs and learns an executor-allocation policy online. This selective-feedback setting relates to costly-observation bandits, which separate action selection from paying to observe its reward~\citep{tucker2023costly,schur2026price}. JOVE extends this perspective to DAGs by jointly selecting executors and outputs to verify under a shared execution-verification budget and per-query latency constraints. 
% Costly-observation bandits explicitly distinguish taking an action from paying to observe its reward~\citep{tucker2023costly,schur2026price}, providing a foundation for balancing execution and learning.
% JOVE brings this perspective to DAG allocation, jointly selecting heterogeneous LLM executors and outputs to verify under a shared execution and verification budget and a separate latency constraint for each query.
\section{Method}
\label{sec:method}
In this section, we formulate the joint execution and verification problem (\secref{sec:method-problem}) and derive a per-query oracle MILP assuming known executor statistics (\secref{sec:method-milp}). We then develop an online policy replacing these statistics with estimates and adding an information-gain bonus for verifications (\secref{sec:method-estimation}). Finally, we establish a sublinear quality-learning regret guarantee (\secref{sec:method-theory}).

\subsection{Problem formulation}
\label{sec:method-problem}
Let $x_t$ denote the $t$-th query, where $t=1,2,\ldots$ indexes query arrivals.
A planner LLM decomposes $x_t$ into a DAG $\mathcal G_t=(\mathcal N_t,\mathcal E_t)$, where each node $i\in\mathcal N_t$ is a subtask. An edge $(j,i)\in\mathcal E_t$ indicates that node $i$ requires an output from node $j$ before execution. 
Each node specifies a subtask instruction, required upstream outputs, a success criterion, and a candidate executor set $\mathcal A_t^i$.
The planner decides $\mathcal A_t^i$ based on the subtask's needs and each model's capabilities, supported modalities, and context length; thus, candidate sets may differ across nodes. A designated output node produces the final response. All executor and verification decisions are made before graph execution.
Appendix~\ref{app:notation} summarizes the notation used throughout the paper.
% Let $\mathcal H_t$ denote the information available at round $t$: past observations, the current query and graph, and the resource constraints.

\textbf{Executor selection and verification.}
% Let $\mathcal A$ denote the set of APIs available for subtask execution. 
For node $i$ of query $x_t$, an online policy $\pi$ selects the joint action $(a_t^i,h_t^i)\in\mathcal A_t^i\times\{0,1\}$, where $a_t^i$ is the executor model and $h_t^i=1$ selects its output for verification\footnote{We leave the dependence of $a_t^i$ and $h_t^i$ on \(\pi\) implicit for notational simplicity.}. Executing node $i$ with model $a$ incurs stochastic API cost $M_t^i(a)$ and latency $D_t^i(a)$.
Its latent correctness $Y_t^i(a)\in\{0,1\}$ indicates whether the output satisfies the node's success criterion and is not observed from execution alone. 
If $h_t^i=1$, a fixed verifier LLM (or human judge), shared across nodes, assesses the subtask, output, and criterion at cost $\widetilde M_t^i$. It returns a binary label $\widehat Y_t^i\in\{0,1\}$ estimating $Y_t^i(a_t^i)$. Verification informs future allocations but does not alter the current response. We assume that verification proceeds asynchronously without delaying graph execution or response delivery. Its cost counts toward the API budget, whereas its latency is excluded from the current query's response time.
% All selected labels and realized costs are observed before allocating the next query.

\textbf{Graph-level cost and latency.} Let $\mathbf a_t=\{a_t^i\}_{i\in\mathcal N_t}$ and $\mathbf h_t=\{h_t^i\}_{i\in\mathcal N_t}$. The total API cost of processing query \(x_t\) is a random variable:
\begin{equation}
\label{eq:cost}
C_t(\mathbf{a}_t,\mathbf{h}_t)
=M_t^{\mathrm{plan}}+
\sum_{i\in\mathcal{N}_t}
\left(
M_t^i(a_t^i)+h_t^i\widetilde{M}_t^i
\right),
\end{equation}
where \(M_t^{\mathrm{plan}}\) denotes the random API cost of task graph generation. A local planner incurs no API charge, so \(M_t^{\mathrm{plan}}=0\).
Devices with limited computational resources may instead use a remote LLM API planner. Its realized cost \(m_t^{\mathrm{plan}}\) is observed before executor allocation and is therefore constant with respect to the allocation decisions.
Let $d_t^{\mathrm{plan}}$ denote the task graph generation time plus the time required to allocate executor models. Assuming that each node starts as soon as its required upstream outputs are available, with no additional scheduling delays, graph execution latency equals the critical-path duration: $L_t=\max_{P\in\mathcal P(\mathcal G_t)}\sum_{i\in P}D_t^i(a_t^i)$, where $\mathcal P(\mathcal G_t)$ is the set of source-to-sink paths. Provided that executor allocation is completed within its reserved time, the end-to-end inference latency is therefore at most
\(
d_t^{\mathrm{plan}}+L_t.
\)

\textbf{Objective and constraints.}
Let $\mathcal{H}_t$ denote the information available before allocating query $x_t$: past allocations with their realized costs, latencies, and verification labels, together with the current graph $\mathcal{G}_t$ and the realized planning cost and overhead.
We maximize the long-run average fraction of correctly executed nodes per query, subject to a long-run average API-cost budget $\gamma$ per query and a deadline $\mu_t$ with violation tolerance $\delta_t\in(0,1)$ for each query:
\begin{subequations}\label{eq:main-problem}
\begin{align}
\max_\pi\quad &\liminf_{T\to\infty}\frac{1}{T}\sum_{t=1}^T
 \mathbb E_\pi\!\left[\frac{1}{|\mathcal{N}_t|}\sum_{i\in\mathcal N_t}Y_t^i(a_t^i)\right]
 \label{eq:true-obj}\\
\mathrm{s.t.}\quad &\limsup_{T\to\infty}\frac1T\sum_{t=1}^T
 \mathbb E_\pi[C_t(\mathbf{a}_t,\mathbf{h}_t)]\leq\gamma,\label{eq:budget}\\
&\Pr_{\pi}\!\left(L_t\leq\mu_t-d_t^{\mathrm{plan}}\,\middle|\,\mathcal{H}_t,\mathbf{a}_t\right)\geq1-\delta_t,
 \quad\forall t,\ \forall\mathcal{H}_t,\ \forall \mathbf a_t \label{eq:latency-chance}\\
&a_t^i\in\mathcal A_t^i,\quad h_t^i\in\{0,1\},
 \quad\forall i\in\mathcal N_t,\ \forall t.\label{eq:decision-domain}
\end{align}
\end{subequations}
\textbf{Interpretation.}
The objective is an additive proxy for the quality of the final-answer: it encourages reliable intermediate results, and, under a sound decomposition of the original query, correct execution of all subtasks is sufficient for a correct final response. The budget permits higher spending on difficult queries to be offset by savings on others. Latency is constrained per query because a fast response cannot compensate for another missing its deadline; conditioning \eqref{eq:latency-chance} on the realized history and selected allocation enforces this at every round.
\begin{remark}[Role of $h_t^i$]
If expected model qualities were known, $h_t^i=0$ would be optimal since verification adds cost without changing the response.
When the qualities are unknown (\secref{sec:method-estimation}), setting $h_t^i=1$ provides a label $\widehat Y_t^i$ for learning model quality and informing future allocations.
\end{remark}

% Here, \(\gamma>0\) is the average API budget per query, while \(\mu_t\) and \(\delta_t\in(0,1)\) denote the deadline and maximum violation probability for query \(t\), respectively. The expectation and probability are induced by the policy and the stochastic API outcomes.
% The node-level objective is a stronger surrogate for final-response quality: every failed subtask is penalized, and, under sound decomposition and aggregation, correct execution of all subtasks is sufficient for a correct final response.
% The time-average budget allows higher spending on difficult queries to be offset by lower spending on easier ones. In contrast, latency is enforced per query because one fast response cannot compensate for another missing its deadline, and different queries may have different \((\mu_t,\delta_t)\) requirements.

% API quality, cost, and latency are stochastic and unknown before execution. To separate constraint handling from online learning, \secref{sec:method-milp} first develops the oracle per-query optimization assuming their conditional statistics are known. Section~\ref{sec:method-estimation} then estimates these quantities online, using all executions for cost and latency feedback and selective verification for node-quality feedback.

\subsection{Oracle MILP Formulation}
\label{sec:method-milp}

% \textbf{Oracle coefficients and joint encoding.}
% We first consider an oracle setting in which model quality, cost, and
% latency statistics are known. To encode executor and verification decisions
% jointly, define
% \(
% z_t^i(a,\eta)
% =
% \mathbbm{1}\!\left\{a_t^i=a,\;h_t^i=\eta\right\}, \text{ and }
% \sum_{a\in\mathcal A_t^i}\sum_{\eta\in\{0,1\}}z_t^i(a,\eta)=1,
% \forall i\in\mathcal N_t.
% \)
% Let
% \(\bar y_t^i(a)=\mathbb E[Y_t^i(a)]\), and define the
% expected execution and verification cost at node $i$ as 
% \begin{equation}
% \label{eq:joint-expected-cost}
% \bar c_t^i(a,\eta)
% =m_t^{\mathrm{plan}}+
% \mathbb E\!\left[
% M_t^i(a)+\eta\widetilde M_t^i
% \right].
% \end{equation}

\textbf{Oracle coefficients and joint encoding.}
We first consider an oracle setting in which model quality, cost,
and latency statistics are known.
To encode executor and verification decisions jointly, define
$z_t^i(a,\eta)
=\mathbbm{1}\!\left\{a_t^i=a,\;h_t^i=\eta\right\}$, where $\sum_{a\in\mathcal A_t^i}\sum_{\eta\in\{0,1\}}z_t^i(a,\eta)=1$,
$\forall i\in\mathcal N_t$.
Define the expected execution and verification cost
at node $i$ as
\(
\bar c_t^i(a,\eta)
=
\mathbb E\!\left[
M_t^i(a)+\eta\widetilde M_t^i
\right].
\)
The expected total cost is
\(
m_t^{\mathrm{plan}}
+\sum_{i\in\mathcal N_t}
\sum_{a\in\mathcal A_t^i}
\sum_{\eta\in\{0,1\}}
\bar c_t^i(a,\eta)z_t^i(a,\eta).
\)

\textbf{Long-term budget via dynamic pricing.}
We associate~\eqref{eq:budget} with a nonnegative multiplier $\lambda_t$, interpreted as the current shadow price of API spending. Starting from $\lambda_1=0$, we update the price after observing query $t$'s realized total cost $c_t$, including planning, execution, and verification:
\begin{equation}
\label{eq:dual-update}
\lambda_{t+1}
=
\left[\lambda_t+\alpha_\lambda(c_t-\gamma)\right]_+,
\qquad \lambda_1=0,
\end{equation}
where \(\alpha_\lambda>0\) is the step size and $[u]_+=\max\{u,0\}$. Spending above $\gamma$ increases the penalty on API costs, whereas spending below decreases it, subject to $\lambda_t \geq 0$. 
Let $\bar y_t^i(a)=\mathbb{E}[Y_t^i(a)]$. 
For query $t$, the price computed from preceding queries enters the objective as
\begin{equation}
\label{eq:priced-objective}
\max_{\mathbf z_t} 
\sum_{i\in\mathcal N_t}\sum_{a\in\mathcal A_t^i}
\sum_{\eta\in\{0,1\}}
\left[
\frac{1}{|\mathcal{N}_t|}\bar y_t^i(a)-\lambda_t\bar c_t^i(a,\eta)
\right]z_t^i(a,\eta).
\end{equation}
% Under the oracle and feasibility assumptions in Appendix~\ref{app:dual-proof}, maximizing~\eqref{eq:priced-objective} over latency-feasible assignments and updating $\lambda_t$ according to \eqref{eq:dual-update} ensures the long-term budget constraint.
Appendix~\ref{app:dual-proof} establishes a long-term budget guarantee for the oracle policy under the stated environment and feasibility assumptions.

\textbf{Per-query latency via node-level quantiles.}
To enforce \eqref{eq:latency-chance}, we split the violation tolerance $\delta_t$ across all execution nodes, choosing $\delta_t^i>0$ with $\sum_{i\in\mathcal N_t}\delta_t^i\leq\delta_t$.
For each node and candidate executor, define the smallest deadline met with probability at least $1-\delta_t^i$ as
\begin{equation}
\label{eq:latency-quantile}
\bar d_t^i(a)
=
\ell_t^i(a,\delta_t^i)
=
\inf\!\left\{
d\geq0:
\Pr\!\left(D_t^i(a)\leq d | \mathcal{H}_t \right)
\geq1-\delta_t^i
\right\}.
\end{equation}
% We use the uniform allocation
% \(\delta_t^i=\delta_t/|\mathcal N_t|\) by default. 
A union bound then gives the following sufficient condition for \eqref{eq:latency-chance}:
\begin{align}
\!\!\!\!\max_{P\in\mathcal P(\mathcal G_t)}
\sum_{i\in P}\sum_{a\in\mathcal A_t^i}\!\!\bar d_t^i(a)\!\!\!\!\!\sum_{\eta\in\{0,1\}}\!\!
\!z_t^i(a,\eta)
\!\leq\!\mu_t\!-\!d_t^{\mathrm{plan}}
\Rightarrow
\Pr\!\left(
L_t\leq\mu_t-d_t^{\mathrm{plan}}|\mathcal{H}_t,\mathbf{a}_t
\right)\!\!
\geq \!\!1\!-\!\delta_t.\label{eq:path-quantile}
\end{align}
Appendix~\ref{app:latency-proof} provides the proof. This reduction may be conservative, but requires no independence among node
latencies. We use $\delta_t^i=\delta_t/|\mathcal N_t|$ by default.
Appendix~\ref{app:risk-allocation} considers a less restrictive allocation that jointly selects executors and node risk $\delta_t^i$ from a discrete grid of risk levels. 
% Our ablation shows little performance difference from uniform risk allocation.

\textbf{Oracle MILP.}
\Eqref{eq:priced-objective} adds a cost penalty to the original objective (\eqref{eq:true-obj}), yielding a linear objective in the allocation variables. 
To formulate the oracle MILP, we express the deterministic latency condition in \eqref{eq:path-quantile} as a set of linear constraints, by introducing finish-time variables $F_t^i$.
Let $\mathcal S_t$ and $\mathcal T_t$ denote the sets of source and sink nodes, respectively.
Define $\mathbf{z}_t^i=\{z_t^i(a,\eta)\}_{a\in\mathcal{A}_t^i,\eta\in\{0,1\}}$ and let
$v_t^i(\mathbf{z}_t^i)=\sum_{a\in\mathcal A_t^i} \bar d_t^i(a)\sum_{\eta\in\{0,1\}} z_t^i(a,\eta)$ be the selected latency quantile for node $i$. The deterministic path-length condition (\eqref{eq:path-quantile}) holds if and only if 
the following inequalities are feasible:
% there exist finish-time variables satisfying
\begin{equation}\label{eq:finish-time-constraints}
\begin{aligned}
 F_t^i \geq v_t^i(z), i\in\mathcal S_t,
 \quad F_t^i \geq F_t^j+v_t^i(z),
 (j,i)\in\mathcal E_t,
 \quad F_t^i \leq\mu_t-d_t^{\mathrm{plan}},
 i\in\mathcal T_t.
\end{aligned}
\end{equation}
These inequalities account for task dependencies and require every path to finish by the deadline; together with the binary assignment constraints and~\eqref{eq:priced-objective}, they define a per-query MILP. The complete formulation is given in Appendix~\ref{app:oracle-milp}. For later use, let \(\mathbf d_t=\{\bar d_t^i(a)\}_{i,a}\) and let \(\mathcal F_t(\mathbf d_t)\) denote the resulting feasible set of assignment and finish-time variables. Since the oracle already knows model qualities, verification provides no learning benefits; thus, choosing $\eta=0$ at every node is optimal. We retain $\eta$ to accommodate verification when model qualities are unknown (\secref{sec:method-estimation}).

\subsection{Online estimation and verification-aware allocation}
\label{sec:method-estimation}
The oracle optimization requires expected model correctness \(\bar y_t^i(a)\), expected execution and verification cost \(\bar c_t^i(a,\eta)\), and latency quantiles \(\bar d_t^i(a)\), which are unknown in the online setting. Execution directly reveals cost and latency, whereas selective verification provides correctness labels for a subset of nodes. We use these observations to estimate the coefficients in an online fashion and augment the MILP objective (\eqref{eq:priced-objective}) with an information-gain bonus for acquiring verification feedback.

\textbf{Model quality estimation.}
% For each task-API pair, a lightweight embedding model (\texttt{Qwen3-Embedding-0.6B}) maps its task description, dependency
% context, and API description to a feature vector
% \(\phi_t^i(a)\in\mathbb R^d\). 
For each task-model pair, we embed the task instruction, descriptions of required node outputs, and candidate model description into a feature vector $\phi_t^i(a)\in\mathbb R^d$ using the embedding model specified in Appendix~\ref{app:embedding}.
Let $(\phi,\hat y)\in \mathcal D_{t-1}$ denote the feature--label pairs obtained through verification before query $t$. Ridge regression with regularization $\lambda_{\mathrm{reg}}>0$ yields
\begin{equation}\label{eq:ridge}
 G_{t-1}=\lambda_{\mathrm{reg}}I+
 \sum_{(\phi,\hat y)\in\mathcal D_{t-1}}\phi\phi^\top,
 \qquad
 \hat\theta_t=G_{t-1}^{-1}
 \sum_{(\phi,\hat y)\in\mathcal D_{t-1}}\phi\hat y.
\end{equation}
Using a linear upper confidence bound (LinUCB) construction  ~\citep{abbasi2011improved}, we define the uncertainty score and model quality estimate as
\begin{equation}\label{eq:ucb}
 u_t^i(a)=\sqrt{\phi_t^i(a)^\top G_{t-1}^{-1}\phi_t^i(a)},
 \quad
 \mathrm{UCB}_t^i(a)=\mathrm{clip}
 (\phi_t^i(a)^\top\hat\theta_t+\beta_tu_t^i(a),0,1),
\end{equation}
where $\beta_t$ scales the uncertainty bonus and
$\mathrm{clip}(\cdot,0,1)$ restricts the estimate to $[0,1]$. 
Past verified observations reduce uncertainty along the feature directions they cover, supporting more accurate quality predictions. $\beta_tu_t^i(a)$ favors underexplored task--model pairs.
Appendix~\ref{app:regret-proof} specifies $\beta_t$ and establishes the high-probability confidence bound used in the regret analysis.

\textbf{Resource estimation.}
Unlike correctness, cost and latency are observed directly from each executor call, so their estimates can be updated without additional feedback cost. The UCB score encourages selecting underexplored task-model pairs, helping resource estimates for each executor $a$ cover a broader range of input contexts.
We reuse the task-model features $\phi$ to estimate expected costs and latency quantiles, weighting past observations by their similarity to the current feature vector $\phi_t^i(a)$.
This yields the estimated execution-plus-verification cost $\hat c_t^i(a,\eta)$ and node-level latency quantile $\hat d_t^i(a)$ used in the online MILP.
Appendix~\ref{app:estimators} provides the estimation details.

\textbf{Verification-aware allocation.}
We replace the unknown coefficients in the oracle MILP (\secref{sec:method-milp}) with the above estimates. Execution reveals cost and latency, but model quality feedback requires paid verification. Since verification does not alter the current response, we reward its information gain to support future executor selection. Specifically, verifying a task-model pair $(i,a)$ updates the information matrix to \(G_{t-1}^{+}=G_{t-1}+\phi_t^i(a)(\phi_t^i(a))^\top\). We quantify the information gain from this observation using the D-optimal criterion \citep{de1995d},
\begin{equation}
\label{eq:dopt}
I_t^i(a)
=
\frac{1}{2}
\log
\frac{\det(G_{t-1}^{+})}{\det(G_{t-1})}
=
\frac{1}{2}\log\!\left(1+u_t^i(a)^2\right),
\end{equation}
\textbf{Interpretation.}
The eigenvalues of $G_{t-1}$ measure accumulated information, including regularization, along its eigenvector directions. 
Since the determinant is the product of these eigenvalues, the gain $I_t^i(a)$ is large when the candidate feature vector has a substantial component in a direction with limited prior observations, and small when those representations are well represented. For a fixed confidence radius $\beta_t$, this gain equals the reduction in the log-volume of the parameter confidence ellipsoid after adding the observation. For the online MILP, JOVE replaces the oracle coefficients with online estimates and adds a verification bonus:
\begin{equation}
\label{eq:online-milp}
\max_{(\mathbf z,\mathbf F)\in\mathcal F_t(\hat{\mathbf d}_t)}
\quad
\sum_{i\in\mathcal N_t}
\sum_{a\in\mathcal A_t^i}
\sum_{\eta\in\{0,1\}}
\left[
\frac{1}{|\mathcal{N}_t|}(\mathrm{UCB}_t^i(a)
+
k_v\eta I_t^i(a))
-
\lambda_t\hat c_t^i(a,\eta)
\right]
z_t^i(a,\eta).
\end{equation}
The three terms represent estimated model quality, verification gain, and priced API cost, respectively. 
The feasible set $\mathcal F_t(\hat{\mathbf d}_t)$, defined in \secref{sec:method-milp}, enforces assignment constraints and the latency condition using estimated node durations. The budget price $\lambda_t$ adapts to observed spending, while $k_v\geq0$ controls the weights assigned to verification gain. 
% As verified labels accumulate along a feature direction, its information gain decreases, making further verification less attractive even at fixed $k_v$.
After each query, we update the quality model using verified labels (\eqref{eq:ridge}), the resource estimates using observed calls (\eqref{eq:empirical-resource-estimates} and \ref{eq:weighted-latency-quantile}), and $\lambda_t$ using the realized total cost (\eqref{eq:dual-update}).
Appendix~\ref{app:algorithm} presents the complete JOVE loop.

\begin{remark}[Verification as a threshold decision]\label{remark:verification-rule}
For a fixed executor, verification leaves latency feasibility and quality term unchanged. It is therefore optimal to set $h_t^i=\mathbf1\{{k_v I_t^i(a_t^i)}/{|\mathcal{N}_t|}\geq\lambda_t\hat{\widetilde m}_t^i\}$.
Optimizing over $h_t^i$ is equivalent to adding $[k_v I_t^i(a)/|\mathcal{N}_t|-\lambda_t\hat{\widetilde m}_t^i]_+$
to each candidate executor's score, so the value of verification also influences executor selection. JOVE can prefer an executor offering valuable future feedback over one with a higher immediate quality score.
\end{remark}

\subsection{Regret Analysis}
\label{sec:method-theory}
In this section, we establish sublinear quality-learning regret for a range of verification weights $k_v$ under the following assumptions.

% Sep 21: clarify graph and latency stationary or iid assumptions
% how latency is set (by the user, or by the planner)
% mu_t, delta_t in assumptions to be positive
% reference model: aggressive way, set the cost and latency using the reference model on a simple dataset, notice that Qwen3-32B is reasoning model that can enable reasoning automatrically.

\begin{assumption}[Quality and feedback]
\label{ass:quality}
There exists a fixed parameter $\theta^\star$ such that, for every available task-model pair, $\bar y_t^i(a)=\phi_t^i(a)^\top\theta^\star\in[0,1]$,
where $\|\phi_t^i(a)\|_2\leq L$ and $\|\theta^\star\|_2\leq S$.
Verified labels satisfy $\widehat Y_t^i=\bar y_t^i(a_t^i)+\varepsilon_t^i$. Conditioned on the selected features and all preceding observations, $\varepsilon_t^i$ has zero-mean and is $R$-sub-Gaussian.
Selected verification labels and realized query costs are available before the next query is allocated. 
\end{assumption}
\begin{assumption}[Allocation regularity]
\label{ass:regret}
For all $t\leq T$,
$1\leq|\mathcal N_t|\leq n_{\max}$,
$0\leq\lambda_t\leq\lambda_{\max}$,
$0\leq\hat{\widetilde m}_t^i\leq\widetilde m_{\max}$, and $\lambda_{\mathrm{reg}}\geq L^2$.
For every $t\leq T$, the estimated feasible set $\mathcal F_t(\hat{\mathbf d}_t)$ is nonempty, and the online MILP (\eqref{eq:online-milp}) is solved to optimality.
\end{assumption}

Assumption~\ref{ass:quality} formalizes a stationary quality model and unbiased feedback. Assumption~\ref{ass:regret} bounds priced verification costs and requires optimal feasible allocations. 
Under these conditions, we compare JOVE with the following benchmark policy.

\textbf{Benchmark policy.}
To isolate quality learning, we compare JOVE with a policy that knows the true expected qualities $\bar y_t^i(a)$ but uses the same resource estimates and budget prices as JOVE.
We define
\begin{equation}\label{eq:quality-regret-objective}
 J_t(\mathbf{z}_t)=\sum_{i\in\mathcal N_t}\sum_{a\in\mathcal A_t^i}
\sum_{\eta\in\{0,1\}}
 \bigl[\frac{1}{|\mathcal{N}_t|}\bar y_t^i(a)-\lambda_t\hat c_t^i(a,\eta)\bigr]
 z_t^i(a,\eta),
 \quad
(\mathbf z_t^\star,\mathbf F_t^\star)
\in \!\!
\operatorname*{arg\,max}_{(\mathbf z,\mathbf F)
\in\mathcal F_t(\hat{\mathbf d}_t)} \!\!\!
J_t(\mathbf z).
\end{equation}
The quality-learning regret is
$\mathcal R_T=\sum_{t=1}^T[J_t(\mathbf{z}_t^\star)-J_t(\mathbf{z}_t)]$, where $\mathbf z_t$ is JOVE's selected allocation.
% To isolate the effect of selective verification on quality learning, we use the same resource estimates and budget prices for JOVE and the comparator policies, accounting for execution and verification costs while keeping resource-estimation error outside the scope of our analysis.

\begin{theorem}[A sufficient verification-weight interval]
\label{thm:sublinear-regret}
Under Assumptions~\ref{ass:quality}-\ref{ass:regret}, fix $\delta_{\mathrm{conf}}\in(0,1)$ and a horizon $T\geq 8d\log (8n_{\mathrm{max}})$, where $d$ is the dimension of the feature vector $\phi_t^i(a)\in\mathbb{R}^d$. Let $\beta_t$ be specified in Appendix~\ref{app:regret-proof}, and define $c_v=\lambda_{\max}\widetilde m_{\max}>0$,
\(
 C_{\mathrm{bat}}
 =\left(1+\frac{(n_{\max}-1)L^2}{\lambda_{\mathrm{reg}}}\right)
  \left(1+\frac{L^2}{\lambda_{\mathrm{reg}}}\right)
  \), and
\( 
  \Gamma_T=C_{\mathrm{bat}}d
 \log\!\left(1+\frac{n_{\mathrm {max}}TL^2}{d\lambda_{\mathrm{reg}}}\right).
\)
For any fixed $k_v\in[K_T,2K_T]$, where
$ K_T=\frac{2n_{\max}c_v}
{\log \! \left(1+\left(\Gamma_T/(n_{\max}T)\right)^{2/3}\right)}$, with probability at least $1-\delta_{\mathrm{conf}}$,
\begin{equation}\label{eq:regret-bound}
 \mathcal R_T
 \leq 4(n_{\max}c_v+\max_{t\leq T}\beta_t)(n_{\mathrm{max}}T)^{2/3}\Gamma_T^{1/3}.
\end{equation}
For fixed $d,\delta_{\mathrm{conf}}$, $n_{\max}$, and the
remaining assumption constants, this bound is
$O(T^{2/3}[\log(1+T)]^{5/6})$,
which is sublinear in $T$ (see Appendix~\ref{app:regret-proof} for proof).
\end{theorem}

\textbf{Interpretation.}
The theorem identifies an interval of verification weights yielding sublinear quality-learning regret under the stated assumptions. 
The lower bound $K_T$ ensures verification of sufficiently uncertain task-model pairs, while the upper bound $2K_T$ keeps the cumulative bonus contribution sublinear in $T$. 
This horizon-dependent interval is sufficient for sublinear regret, but does not identify an optimal $k_v$.
The condition $T\geq8d\log(8n_{\max})$ is conservative.
Reducing $d$ lowers this threshold and may enable faster learning, but also restricts the quality model's expressiveness, potentially reducing prediction accuracy.
Empirically, JOVE performs well across a broad range of verification weights, suggesting limited sensitivity to $k_v$ within the tested range.

% The horizon condition $T\geq8d\log(8n_{\max})$ is conservative; a smaller feature dimension lowers this threshold. 
% Empirically, JOVE performs well across a broad range of verification weights, suggesting limited sensitivity to $k_v$ over the tested range.

\section{Experiments}
\label{sec:experiments}

\newcommand{\accstd}[2]{\ensuremath{#1 \pm #2}}

% Latency/cost: smaller standard deviation
\newcommand{\meanstd}[2]{%
  \ensuremath{#1{\scriptstyle\,\pm\,#2}}%
}
\definecolor{lightgray}{rgb}{0.93, 0.93, 0.93}
\begin{table*}[t]
\centering
\caption{
Accuracy, latency, and API cost per query across four benchmarks. Averaged across benchmarks, JOVE achieves competitive accuracy with substantially lower latency and API cost than the unconstrained inference baselines. Values are mean $\pm$ standard deviation over three random seeds.
}
\label{tab:main-results}
\small
\setlength{\tabcolsep}{2pt}
\renewcommand{\arraystretch}{1.12}
\resizebox{0.99\linewidth}{!}{%
\begin{tabular}{@{}l*{8}{c}@{}}
\toprule

% ==================== Accuracy panel ====================
\multicolumn{9}{c}{\textbf{(a) Accuracy (\%) $\uparrow$}} \\
\midrule
Methods
& \multicolumn{2}{c}{Bamboogle}
& \multicolumn{2}{c}{MMLU-Pro}
& \multicolumn{2}{c}{LiveBench-Reasoning}
& \multicolumn{2}{c}{GPQA} \\
\midrule

\multicolumn{9}{@{}l}{
  \textit{Inference-level baselines (unconstrained)}
} \\[2pt]

\rowcolor{lightgray}Direct (Qwen3-32B)
& \multicolumn{2}{c}{\accstd{47.5}{3.2}}
& \multicolumn{2}{c}{\accstd{47.3}{0.8}}
& \multicolumn{2}{c}{\accstd{31.4}{1.2}}
& \multicolumn{2}{c}{\accstd{45.1}{3.8}} \\

CoT~\citep{wei2022cot}
& \multicolumn{2}{c}{\accstd{\textbf{51.7}}{1.2}}
& \multicolumn{2}{c}{\accstd{54.1}{1.1}}
& \multicolumn{2}{c}{\accstd{39.2}{0.9}}
& \multicolumn{2}{c}{\accstd{54.2}{3.6}} \\

SoT~\citep{ning2024sot}
& \multicolumn{2}{c}{\accstd{48.9}{3.4}}
& \multicolumn{2}{c}{\accstd{52.0}{3.1}}
& \multicolumn{2}{c}{\accstd{33.4}{2.3}}
& \multicolumn{2}{c}{\accstd{50.9}{6.8}} \\

Plato~\citep{jin2025plato}
& \multicolumn{2}{c}{\accstd{46.5}{0.9}}
& \multicolumn{2}{c}{\accstd{52.8}{3.2}}
& \multicolumn{2}{c}{\accstd{38.4}{1.8}}
& \multicolumn{2}{c}{\accstd{\textbf{60.3}}{0.1}} \\

\midrule
\multicolumn{9}{@{}l}{
  \textit{Resource-aware methods (constrained)}
} \\[2pt]

WR-Online~\citep{li2026wiserouter}
& \multicolumn{2}{c}{\accstd{43.3}{3.5}}
& \multicolumn{2}{c}{\accstd{45.7}{3.8}}
& \multicolumn{2}{c}{\accstd{38.0}{4.6}}
& \multicolumn{2}{c}{\accstd{48.1}{1.7}} \\

Knapsack-ascend
& \multicolumn{2}{c}{\accstd{48.0}{2.1}}
& \multicolumn{2}{c}{\accstd{41.0}{6.0}}
& \multicolumn{2}{c}{\accstd{32.2}{5.6}}
& \multicolumn{2}{c}{\accstd{39.0}{1.3}} \\

Knapsack-descend
& \multicolumn{2}{c}{\accstd{44.3}{7.0}}
& \multicolumn{2}{c}{\accstd{42.6}{2.9}}
& \multicolumn{2}{c}{\accstd{33.2}{2.0}}
& \multicolumn{2}{c}{\accstd{41.2}{1.2}} \\

\rowcolor{lightgray}JOVE single-node
& \multicolumn{2}{c}{\accstd{28.8}{8.4}}
& \multicolumn{2}{c}{\accstd{46.9}{11.4}}
& \multicolumn{2}{c}{\accstd{28.1}{4.2}}
& \multicolumn{2}{c}{\accstd{43.5}{9.4}} \\

\rowcolor{highlight}\textbf{JOVE}
& \multicolumn{2}{c}{\accstd{49.3}{1.0}}
& \multicolumn{2}{c}{\accstd{\textbf{54.8}}{1.1}}
& \multicolumn{2}{c}{\accstd{\textbf{41.3}}{1.2}}
& \multicolumn{2}{c}{\accstd{52.7}{2.8}} \\

% ================= Latency and cost panel =================
\midrule[\heavyrulewidth]
\multicolumn{9}{c}{
  \textbf{(b) Latency (sec.) and cost ($\mathbf{\times 10^{-4}}$USD) $\downarrow$}
} \\
\midrule[\heavyrulewidth]
% Method
% & \multicolumn{2}{c}{Bamboogle}
% & \multicolumn{2}{c}{MMLU-Pro}
% & \multicolumn{2}{c}{LiveBench-Reasoning}
% & \multicolumn{2}{c}{GPQA} \\
% \cmidrule(lr){2-3}
% \cmidrule(lr){4-5}
% \cmidrule(lr){6-7}
% \cmidrule(l){8-9}
Methods
& Latency & Cost
& Latency & Cost
& Latency & Cost
& Latency & Cost \\
\midrule

\multicolumn{9}{@{}l}{
  \textit{Inference-level baselines (unconstrained)}
} \\[2pt]

\rowcolor{lightgray}Direct (Qwen3-32B)
& \meanstd{\underline{14.09}}{0.99} & \meanstd{\underline{2.00}}{0.15}
& \meanstd{45.39}{2.29} & \meanstd{5.34}{0.38} & \meanstd{41.18}{0.20} & \meanstd{10.54}{0.01} & \meanstd{59.46}{10.64} & \meanstd{10.56}{2.28} \\

CoT~\citep{wei2022cot}
& \meanstd{18.61}{2.27} & \meanstd{2.24}{0.41}
& \meanstd{34.11}{4.63} & \meanstd{5.38}{0.34} & \meanstd{80.38}{17.48} & \meanstd{12.28}{1.31} & \meanstd{101.11}{14.00} & \meanstd{12.81}{1.96} \\

SoT~\citep{ning2024sot}
& \meanstd{17.04}{9.71} & \meanstd{10.56}{0.25}
& \meanstd{70.17}{6.27} & \meanstd{27.53}{3.71} & \meanstd{129.28}{23.46} & \meanstd{52.62}{0.02} & \meanstd{137.56}{37.40} & \meanstd{55.46}{4.72} \\

Plato~\citep{jin2025plato}
& \meanstd{43.95}{4.62} & \meanstd{8.89}{0.40}
& \meanstd{104.20}{5.21} & \meanstd{22.23}{2.02} 
& \meanstd{189.51}{41.48} & \meanstd{40.17}{1.22} 
& \meanstd{145.7}{36.94} & \meanstd{37.35}{6.57} \\

\midrule
\multicolumn{9}{@{}l}{
  \textit{Resource-aware methods (constrained)}
} \\[2pt]

WR-Online~\citep{li2026wiserouter}
& \meanstd{16.02}{2.97} & \meanstd{1.96}{0.10}
& \meanstd{21.17}{1.29} & \meanstd{2.07}{0.02}
& \meanstd{23.08}{6.74} & \meanstd{2.04}{0.13}
& \meanstd{17.25}{0.28} & \meanstd{2.09}{0.09} \\

Knapsack-ascend
& \meanstd{4.66}{0.20} & \meanstd{1.82}{0.03}
& \meanstd{4.70}{0.79} & \meanstd{1.63}{0.25}
& \meanstd{16.68}{6.85} & \meanstd{2.87}{1.18}
& \meanstd{5.65}{0.63} & \meanstd{2.34}{0.40} \\

Knapsack-descend
& \meanstd{6.17}{0.69} & \meanstd{1.97}{0.44}
& \meanstd{5.41}{1.17} & \meanstd{1.72}{0.11}
& \meanstd{16.56}{6.84} & \meanstd{2.37}{1.30}
& \meanstd{6.84}{1.36} & \meanstd{1.88}{0.25} \\

\rowcolor{lightgray}JOVE single-node
& \meanstd{2.53}{1.10} & \meanstd{0.66}{0.12}
& \meanstd{2.99}{0.43} & \meanstd{1.60}{0.14}
& \meanstd{1.15}{0.21} & \meanstd{2.30}{0.23}
& \meanstd{1.66}{0.60} & \meanstd{2.15}{0.10} \\

\rowcolor{highlight}\textbf{JOVE}
& \meanstd{3.68}{0.48} & \meanstd{2.18}{0.18}
& \meanstd{6.66}{2.09} & \meanstd{2.03}{0.16}
& \meanstd{8.71}{1.15} & \meanstd{2.26}{0.43}
& \meanstd{9.32}{0.70} & \meanstd{2.49}{0.14} \\

\bottomrule
\end{tabular}}
\vspace{-0.2cm}
\end{table*}

\textbf{Datasets and metrics.}
We evaluate JOVE on four reasoning benchmarks: Bamboogle~\citep{press2023measuring} (multi-hop QA), MMLU-Pro~\citep{wang2024mmlu} (multidisciplinary reasoning), GPQA~\citep{rein2023gpqa} (science), and LiveBench-Reasoning~\citep{white2025livebench} (structured reasoning). We repeat each experiment over three random orderings of the online query sequence, and measure final answer accuracy, end-to-end response latency, and total API cost per query.
% We report the mean and standard deviation of each metric across the three runs.

\textbf{Experimental settings.}
We use Gemini-2.5-Flash-Lite as the planner and Qwen3-235B-A22B-2507 as the verifier. 
The executor pool includes eight models of different sizes (detailed in Appendix~\ref{app:extra-experiments}), accessed through the \cite{openrouter} platform. 
Unless otherwise specified, we set the verification weight to $k_v=5.0$, exploration coefficient to $\beta_t=0.25$, regularization parameter to $\lambda_{\mathrm{reg}}=1$, budget-update step size to $\alpha_\lambda=100$, and latency violation tolerance to $\delta_t=0.1$.
We use Qwen3-32B as the reference model for accuracy, cost, and latency comparisons.
Based on its average cost and latency on Bamboogle, we set the API-cost budget to $\gamma=2.0\times10^{-4}$ USD and the deadline to $\mu_t=16$s.
We choose Bamboogle because it yields the lowest cost and latency across the four datasets, resulting in stricter resource targets than calibration across all datasets.
% Using the average cost and latency of its calls on Bamboogle as a reference, we set the API-cost budget to $\gamma=2.0\times 10^{-4}$ USD and the end-to-end latency deadline to $\mu_t=16$ seconds. 

\textbf{Baselines.}
We compare against two groups of methods.
(1) \emph{Inference-level baselines}: Direct invocation, Chain-of-Thought (CoT)~\citep{wei2022cot}, Skeleton-of-Thought (SoT)~\citep{ning2024sot}, and Plato~\citep{jin2025plato}. We evaluate these methods using the reference model without budget and latency constraints.
%These baselines span direct generation, sequential reasoning, and structured parallel inference, which help assess whether our constrained allocation JOVE remains competitive with well-established inference methods.
(2) \emph{Resource-aware baselines}: 
% WR-Online~\citep{li2026wiserouter}, Knapsack-ascend, Knapsack-descend, and JOVE single-node. 
\textit{WR-Online}~\citep{li2026wiserouter} routes queries under the cost budget (without latency constraints) by learning cost and reward online as a constrained contextual bandit.
\textit{Knapsack-ascend} selects an allocation that fully utilizes the cost budget $\gamma$ and latency allowance, while ensuring nondecreasing expected execution costs along each source-to-sink path.
\textit{Knapsack-descend} follows the same procedure but instead requires nonincreasing costs. As an ablation, we evaluate \textit{JOVE single-node}, which retains JOVE's allocation and verification mechanisms but answers
each query with one executor call, without decomposition.
For resource-aware methods, we generate four distinct task graphs per query with 2--5 nodes and shuffle all query--graph pairs into an online learning sequence.
We further discuss the learning sequence in Appendix~\ref{app:sequence}.
We did not include resource-aware methods that require offline data~\citep{dong2026hybridflow,qian2026servingagenticworkflowsphysicalplan}, since our setting provides no such data and the policy is learned online.

\subsection{Experimental results}
\textbf{JOVE performance.}
Table~\ref{tab:main-results} reports accuracy, latency, and cost under the budget and latency deadline calibrated from direct Qwen3-32B calls on Bamboogle. 
Averaged over the four benchmarks, JOVE achieves comparable accuracy (49.5\%) at 2.24$\times10^{-4}$~USD and 7.09s per query. Compared with the three inference baselines (CoT, SoT, and Plato), JOVE uses 3.7-16.3$\times$ less cost and 8.3-17.0$\times$ less latency while achieving comparable average accuracy. 
\begin{wrapfigure}{r}{0.40\textwidth}
    \centering
\vspace{-0.2cm}
    \includegraphics[width=0.42\textwidth]{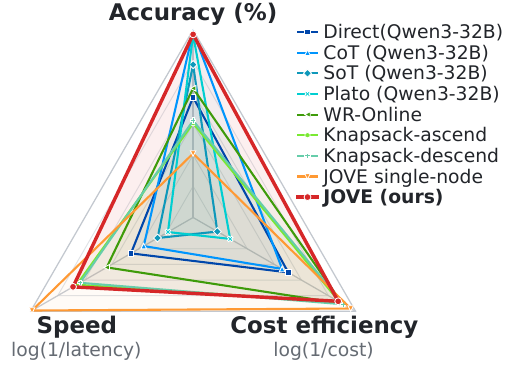}
    \caption{Accuracy, latency, and cost trade-offs, averaged over four datasets.}
    \label{fig:radar}
\vspace{-0.5cm}
\end{wrapfigure}
WR-Online achieves 43.8\% average accuracy but does not enforce a latency constraint; its average latency is 2.7$\times$ that of JOVE.
The constrained heuristics isolate the contribution of the allocation itself: Knapsack-ascend and -descend consume comparable resources yet reach only 40.1\% and 40.3\% on average, 9.5 and 9.2 points below JOVE. Averaged across benchmarks, JOVE single-node is cheapest and fastest (1.68$\times10^{-4}$~USD, 2.08s) but least accurate (36.8\%); it has no task decomposition and receives at most one verification label per query, potentially limiting accuracy and slowing policy learning.
JOVE uses selective node-level verification to learn the model-selection policy and allocates more capable executors where their estimated benefit justifies resource cost.

% \vspace{-0.05in}

Figure~\ref{fig:radar} visualizes the accuracy-latency-cost trade-off reported in Table~\ref{tab:main-results} using the average metrics over the four datasets. The inference-level baselines generally achieve promising accuracy at substantially greater latency and cost than JOVE. The constrained heuristics use comparable resources to JOVE but achieve lower average accuracy, highlighting the value of learning task-dependent executor quality. JOVE achieves a favorable balance across the three metrics by jointly selecting executors and verification decisions under the resource-aware objective.

\begin{figure}[t]
  \centering
  
  \begin{subfigure}{0.45\textwidth}
    \centering
    \includegraphics[width=\linewidth]{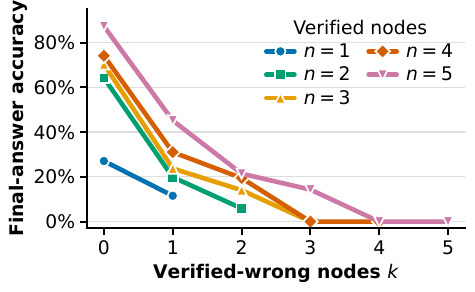}
    \caption{Node-level vs. final-answer correctness}
    \label{fig:nodes-correctness}
  \end{subfigure}
  \hfill % Adds horizontal spacing between the subfigures
  \begin{subfigure}{0.45\textwidth}
    \centering
    \includegraphics[width=\linewidth]{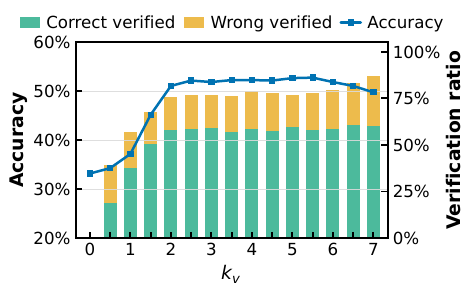}
    \caption{Effect of the verification weight $k_v$}
    \label{fig:kv}
  \end{subfigure}
  \hfill % Adds horizontal spacing between the subfigures
\caption{Verification analysis of JOVE. (a) Final-answer accuracy versus the number of verifier-labeled incorrect nodes $k$, grouped by the number of verified nodes $n$, with results pooled across all four datasets. (b) GPQA final-answer accuracy and the fraction of executed nodes verified, split by verifier-labeled correct and incorrect output, as the verification weight $k_v$
varies.}
  \label{fig:main_total}
  \vspace{-0.3cm}
\end{figure}

\textbf{Node-level correctness vs. final-answer correctness.}
JOVE learns entirely from node-level verification without direct feedback on final response quality, maximizing average node correctness (\eqref{eq:true-obj}). Figure~\ref{fig:nodes-correctness} shows how this additive surrogate relates to overall accuracy by grouping queries by total verified nodes $n \in \{1,\dots,5\}$ and verified-incorrect nodes $k$. The average final-answer accuracy drops monotonically as $k$ increases for every $n$, indicating that node-level errors are associated with lower final-answer accuracy. At fixed $k$, accuracy is generally higher for larger $n$. When all five verified nodes are labeled correct $(n=5,k=0)$, final-answer accuracy reaches 87.3\% rather than 100\%. This gap may reflect verifier noise or imperfect alignment between node-level success criteria and final-answer correctness.

\textbf{Selection of verification weights $k_v$.}
% UCB+kv*I-lambda*c(a,eta)
Figure~\ref{fig:kv} varies $k_v$ on GPQA, reporting final-answer accuracy together with the fraction of executed nodes that are verified, split by whether the verifier judged them correct or incorrect. 
% At $k_v < 1.5$, verification is less frequent and accuracy stays near 30--35\%, because small $k_v$ makes verification cost easily surpass the information gain.  
At $k_v<1.5$, verification is infrequent and accuracy remains relatively low.
A small $k_v$ makes the weighted information gain less likely to outweigh the priced verification cost, limiting the feedback available for quality learning.
For $k_v \in [2, 6]$, accuracy improves and remains stable, indicating that sufficient feedback is available and performance is relatively insensitive to the precise verification weight. With higher $k_v>6$, accuracy starts to decline, potentially because verification consumes a growing share of the budget, leaving fewer resources for execution. This regime also shows an increasing fraction of executed nodes that are verified and labeled incorrect.

\textbf{Satisfaction of the constraints.}
In Section~\ref{sec:method-problem}, we introduce a long-term API-cost budget and a per-query latency chance constraint. 
In practice, JOVE replaces the initially unknown model statistics with online estimates, so the oracle guarantees do not directly extend to the practical policy.
Constraint satisfaction then depends on resource-estimation accuracy and budget-price adaptation. Appendix~\ref{app:constraint-results} evaluates this behavior empirically: as learning progresses, interval-average costs decrease toward or below the budget, while deadline hit rates improve, reaching at least $96\%$ on all four datasets in the final interval against a $90\%$ target.

The appendix provides additional results on the effect of the online learning sequence (Appendix~\ref{app:sequence}), ablations of cost budgets and latency deadlines (Appendix~\ref{app:extra-ablation}), and comparisons of verification strategies (Appendix~\ref{app:verification}).

\section{Conclusion}
We introduce JOVE, an online framework that optimizes the execution of complex reasoning queries represented by task graphs. By jointly managing LLM routing and selective paid verification via per-query MILP, JOVE balances immediate accuracy with feedback acquisition for future allocation decisions. 
We establish sublinear quality-learning regret under the stated assumptions and demonstrate effective budget and deadline control empirically.
JOVE achieves competitive accuracy while substantially reducing cost and latency relative to baselines, providing a promising foundation for deploying efficient, multi-step LLM systems in resource-constrained environments.

\bibliography{iclr2027_conference}
\bibliographystyle{iclr2027_conference}
\newpage
\appendix
\section{Appendix}
% create table of content
\startcontents[appendix]
\printcontents[appendix]{}{1}{\setcounter{tocdepth}{3}}

\subsection{Oracle MILP formulation}
\label{app:oracle-milp}
Let $\mathcal S_t$ and $\mathcal T_t$ denote the source and sink nodes of
$\mathcal G_t$, respectively, and let $F_t^i$ be the finish time of node $i$
relative to the start of graph execution.
Combining the priced objective with the deterministic latency reduction gives
\begin{equation}\label{eq:oracle-milp}
\begin{alignedat}{2}
&\max_{\mathbf{z}_t,\{F_t^i\}_{i\in\mathcal{N}_t}}\quad\sum_{i\in\mathcal{N}_t}\sum_{a\in\mathcal{A}_t^i}\sum_{\eta\in\{0,1\}}
 \left[\frac{1}{|\mathcal{N}_t|}\bar y_t^i(a)-\lambda_t\bar c_t^i(a,\eta)\right]z_t^i(a,\eta)\\
\mathrm{s.t.}\quad
&\sum_{a\in\mathcal{A}_t^i}\sum_{\eta\in\{0,1\}}z_t^i(a,\eta)=1,
 &&\quad\forall i\in\mathcal N_t,\\
&z_t^i(a,\eta)\in\{0,1\},
 &&\quad\forall i,a,\eta,\\
&F_t^i\geq\sum_{a\in\mathcal{A}_t^i}\sum_{\eta\in\{0,1\}}\bar d_t^i(a)z_t^i(a,\eta),
 &&\quad\forall i\in\mathcal S_t,\\
&F_t^i\geq F_t^j+\sum_{a\in\mathcal{A}_t^i}\sum_{\eta\in\{0,1\}}\bar d_t^i(a)z_t^i(a,\eta),
 &&\quad\forall(j,i)\in\mathcal E_t,\\
&F_t^i\leq\mu_t-d_t^{\mathrm{plan}},
 &&\quad\forall i\in\mathcal T_t,\\
&F_t^i\geq0,
 &&\quad\forall i\in\mathcal N_t.
\end{alignedat}
\end{equation}
In the oracle MILP formulation, we assume the LLM expected performance $\bar y_t^i(a)$, expected invocation cost $\bar c_t^i(a,\eta)$ and node-level latency quantiles $\bar d_t^i(a)$ are all known, optimizing the executor selection and verification nodes using $z_t^i(a,\eta)$, and per-node finish time using $F_t^i$. 
In this oracle setting, verification delivers no gain on improving the estimation of LLM API qualities, it only increases the cost as defined in $\bar c_t^i(a,\eta)$.

\subsection{Theoretical Analysis}
\label{app:theory}
Throughout, uppercase letters denote random quantities and lowercase letters their realizations.
\subsubsection{Sufficient condition of the latency constraint}
\label{app:latency-proof}
\begin{lemma}[Sufficient condition of the latency constraint]\label{prop:latency-feasibility}
Fix risk shares $\{\delta_t^i\}_{i\in\mathcal N_t}$ satisfying
$\sum_i\delta_t^i\leq\delta_t$, and define $\bar d_t^i(a)$
by~\eqref{eq:latency-quantile}. For any allocation
determined by $\mathcal{H}_t$,
condition~\eqref{eq:path-quantile} implies the latency chance
constraint~\eqref{eq:latency-chance}.
\end{lemma}
\begin{proof}
Fix the information $\mathcal{H}_t$ and let
$\mathbf a_t=(a_t^i)_{i\in\mathcal N_t}$ be the selected allocation.
Since the allocation is determined by $\mathcal{H}_t$, conditioning
additionally on $\mathbf a_t$ does not change the conditional
probabilities below.

For each node $i\in\mathcal N_t$, define the event
\[
E_t^i
:=
\left\{
D_t^i(a_t^i)>\bar d_t^i(a_t^i)
\right\},
\]
which occurs when its execution latency exceeds its assigned quantile.
By~\eqref{eq:latency-quantile},
\[
\Pr(E_t^i\mid\mathcal{H}_t)\leq\delta_t^i.
\]
The event that at least one node exceeds its assigned quantile is
$\bigcup_{i\in\mathcal N_t}E_t^i$. Applying the conditional union bound,
\[
\Pr\left(
\bigcup_{i\in\mathcal N_t}E_t^i
\,\middle|\,\mathcal{H}_t
\right)
\leq
\sum_{i\in\mathcal N_t}
\Pr(E_t^i\mid\mathcal{H}_t)
\leq
\sum_{i\in\mathcal N_t}\delta_t^i
\leq\delta_t.
\]
Consequently, the complementary event
\[
\mathcal E_t
:=
\bigcap_{i\in\mathcal N_t}(E_t^i)^c
=
\left\{
D_t^i(a_t^i)\leq\bar d_t^i(a_t^i)
\text{ for every }i\in\mathcal N_t
\right\}
\]
satisfies
$\Pr(\mathcal E_t\mid\mathcal{H}_t)\geq1-\delta_t$.

On $\mathcal E_t$, every source-to-sink path
$P\in\mathcal P(G_t)$ satisfies
\[
\sum_{i\in P}D_t^i(a_t^i)
\leq
\sum_{i\in P}\bar d_t^i(a_t^i).
\]
Taking the maximum over all such paths and
using~\eqref{eq:path-quantile} yields
\[
L_t
=
\max_{P\in\mathcal P(G_t)}
\sum_{i\in P}D_t^i(a_t^i)
\leq
\max_{P\in\mathcal P(G_t)}
\sum_{i\in P}\bar d_t^i(a_t^i)
\leq
\mu_t-d_t^{\mathrm{plan}}.
\]
Thus $\mathcal E_t\subseteq
\{L_t\leq\mu_t-d_t^{\mathrm{plan}}\}$, and therefore
\[
\Pr\left(
L_t\leq\mu_t-d_t^{\mathrm{plan}}
\,\middle|\,\mathcal{H}_t,\mathbf a_t
\right)
\geq
\Pr(\mathcal E_t\mid\mathcal{H}_t)
\geq1-\delta_t.
\]
This establishes~\eqref{eq:latency-chance}.
The union bound holds for arbitrarily dependent events, so no
independence assumption among node latencies is required.
\end{proof}

% REWRITTEN: oracle statement strengthened to the actual conservative set.
\subsubsection{Average-budget guarantee}
\label{app:dual-proof}
% \begin{assumption}[Feasible environment]\label{ass:system}
% Queries, together with their task graphs, resource requirements,
% and planning overheads, are drawn independently from the same
% distribution. Moreover, the number of the nodes and cost are bounded for each query:
% $1\leq|\mathcal N_t|\leq n_{\max}$ and
% $0\leq C_t(\mathbf{a}_t,\mathbf{h}_t)\leq C_{\max}$.
% There exists a stationary policy $\pi^0$ that selects allocations $(\mathbf{a}_t^0,\mathbf{h}_t^0)$ satisfying the latency constraint and the average cost meets \(
% \limsup_{T\to\infty}\frac1T
% \sum_{t=1}^T
% \mathbb E[C_t(\mathbf a_t^0,\mathbf h_t^0)]
% \leq\gamma.
% \)
% \end{assumption}

\begin{assumption}[Oracle environment and conservative feasibility]
\label{ass:system}
Let $\xi_t$ denote the current query context, including the
query, task graph, candidate executor sets, resource
requirements, planning overheads, and node risk shares.
The contexts $\{\xi_t\}$ are independent and identically
distributed. The node risk shares are determined by a fixed
rule from the current context.

Conditional on the current context and selected allocation,
the joint distribution of execution latencies and API costs
is time-invariant and does not depend on past rounds.
For every admissible allocation,
$1\leq |\mathcal N_t|\leq n_{\max}$ and
$0\leq C_t(\mathbf a_t,\mathbf h_t)\leq C_{\max}$
almost surely.

There exists a stationary policy $\pi^0$, depending only on
the current context and, if randomized, fresh independent
randomness, whose induced assignment and finish-time variables
satisfy
\[
(\mathbf z_t^0,\mathbf F_t^0)
\in \mathcal F_t(\mathbf d_t)
\qquad\text{almost surely},
\]
where $\mathbf d_t$ contains the true node-level latency
quantiles under the same risk-allocation rule used by the
oracle MILP. Moreover,
\[
\mathbb E_{\pi^0}
\left[C_t(\mathbf a_t^0,\mathbf h_t^0)\right]
\leq \gamma,
\]
where total cost includes planning, execution, and verification.
\end{assumption}

This assumption requires budget feasibility within the
conservative feasible set used by the oracle optimizer, because for JOVE, the node-quantile reduction
may exclude some chance-feasible allocations.
The assumption does not require independence among node
latencies within a query.

\begin{proposition}[Long-term budget feasibility]\label{prop:dual-budget}
Under Assumption~\ref{ass:system}, suppose each query maximizes
\eqref{eq:priced-objective} over
$\mathcal F_t({\mathbf d}_t)$ using true expected costs and qualities. 
Let the multiplier update be \[
\lambda_{t+1}
=
\bigl[\lambda_t+\alpha_\lambda(c_t-\gamma)\bigr]_+,
\qquad \lambda_1=0,
\]
where $c_t$ is the realized cost of solving query $x_t$, and $\alpha_\lambda>0$ controls the tradeoff between performance and cost.
Then $\mathbb E[\lambda_{T+1}]/T\to0$, and the long-term
budget constraint~\eqref{eq:budget} holds:
\[
\limsup_{T\to\infty}
\frac1T\sum_{t=1}^T\mathbb E[C_t(\mathbf{a}_t,\mathbf{h}_t)]\leq\gamma.
\]
\end{proposition}

% REWRITTEN: correct conditioning, common step-size symbol, direct rate proof.
\begin{proof}
Let $ \mathcal{H}_{t-1}$ denote the history before the current context arrives,
so $\lambda_t$ is $\mathcal{H}_{t-1}$-measurable. 
Define the expected average node quality as
\[
U_t(\mathbf a_t)
=
\frac{1}{|\mathcal N_t|}
\sum_{i\in\mathcal N_t}\bar y_t^i(a_t^i),
\qquad
0\leq U_t(\mathbf a_t)\leq1.
\]
Let $(\mathbf a_t^0,\mathbf h_t^0)$ denote the allocation selected
by a feasible policy $\pi^0$, and $(\mathbf a_t^*,\mathbf h_t^*)$ denote the optimal allocation to maximize the priced objective \eqref{eq:priced-objective}.

Since current-query information is identically distributed across
rounds and the conditional cost laws do not change, the expected
cost of policy $\pi_0$ is a constant:
\[
\mathbb E[C_t(\mathbf a_t^0,\mathbf h_t^0)]=m_0.
\]
Its long-term budget feasibility therefore implies
\[
m_0
=
\limsup_{T\to\infty}\frac1T
\sum_{t=1}^T
\mathbb E[C_t(\mathbf a_t^0,\mathbf h_t^0)]
\leq\gamma.
\]
Moreover, $\lambda_t$ depends only on past observations, whereas
the conditional expected cost under $\pi^0$ depends only on
the independent current-query information. Hence
\[
\mathbb E[
\lambda_t C_t(\mathbf a_t^0,\mathbf h_t^0)]
=
m_0\,\mathbb E[\lambda_t].
\]

The planner cost is common to all allocations of the current query,
so including it does not change the oracle optimizer.
Oracle optimality, followed by taking expectations, gives
\[
\mathbb E[
U_t(\mathbf a_t^*)-\lambda_t c_t]
\geq
\mathbb E[
U_t(\mathbf a_t^0)
-\lambda_t C_t(\mathbf a_t^0,\mathbf h_t^0)].
\]
Rearranging and using $0\leq U_t\leq1$, we obtain
\begin{align*}
\mathbb E[\lambda_t(c_t-\gamma)]
&\leq
1+\mathbb E[
\lambda_t C_t(\mathbf a_t^0,\mathbf h_t^0)]
-\gamma\mathbb E[\lambda_t]\\
&=
1+(m_0-\gamma)\mathbb E[\lambda_t]
\leq1.
\end{align*}

Set
\[
B=\frac{\alpha_\lambda}{2}
\max\{\gamma^2,(C_{\max}-\gamma)^2\}.
\]
Using $[u]_+^2\leq u^2$ and $0\leq c_t\leq C_{\max}$,
the multiplier update yields
\[
\frac{\lambda_{t+1}^2-\lambda_t^2}{2\alpha_\lambda}
\leq
\lambda_t(c_t-\gamma)+B.
\]
Taking expectations, summing over $t=1,\ldots,T$, and using
$\lambda_1=0$ gives
\[
\mathbb E[\lambda_{T+1}^2]
\leq2\alpha_\lambda(B+1)T.
\]
Consequently,
\[
\frac{\mathbb E[\lambda_{T+1}]}{T}
\leq
\frac{\sqrt{\mathbb E[\lambda_{T+1}^2]}}{T}
\leq
\sqrt{\frac{2\alpha_\lambda(B+1)}{T}}
\longrightarrow0.
\]

Finally, the update also implies
\[
\lambda_{t+1}-\lambda_t
\geq\alpha_\lambda(c_t-\gamma).
\]
Summing over rounds and taking expectations establishes
\[
\frac1T\sum_{t=1}^T
\mathbb E[C_t(\mathbf a_t^*,\mathbf h_t^*)]
=
\frac1T\sum_{t=1}^T\mathbb E[c_t]
\leq
\gamma+
\frac{\mathbb E[\lambda_{T+1}]}{\alpha_\lambda T}.
\]
Taking the upper limit as $T\to\infty$ proves the long-term
budget constraint.
\end{proof}

\subsubsection{LLM quality learning analysis}
\label{app:regret-proof}
% NEW: make the conditional noise requirement precise for adaptive batches.
\textbf{Selection of $\beta_t$.}
To facilitate the theoretical analysis, we define a standard confidence radius for regularized linear regression \citep{abbasi2011improved} to ensure that the uncertainty bonus
$\beta_t u_t^i(a)$ covers the quality-estimation error with
high probability:
\begin{equation}\label{eq:confidence-radius}
\beta_t
=
\sqrt{\lambda_{\mathrm{reg}}}S
+
R\sqrt{
2\log(1/\delta_{\mathrm{conf}})
+d\log\left(
1+\frac{|\mathcal D_{t-1}|L^2}
{d\lambda_{\mathrm{reg}}}
\right)}.
\end{equation}
The first term accounts for regularization bias, while the
second accounts for noise in the verified feedback.
Under Assumption~\ref{ass:quality}, this $\beta_t$ guarantees the following inequality with probability at least $1-\delta_{\mathrm{conf}}$:
\[
\|\hat\theta_t-\theta^\star\|_{G_{t-1}}
\leq\beta_t
\qquad\text{for all }t,
\]
where $\|v\|_G=\sqrt{v^\top Gv}$. This can be proved following the same steps as in \cite[Theorem~2]{abbasi2011improved}. 
The result applies to our verified observations because their
features are selected before observing the corresponding
conditionally zero-mean, sub-Gaussian label errors $\epsilon_t^i$ as defined in Assumption~\ref{ass:quality}.
Consequently, by Cauchy--Schwarz, we can show:
\[
\left|
\phi_t^i(a)^\top\hat\theta_t-\bar y_t^i(a)
\right|
\leq\beta_t u_t^i(a).
\]
$\beta_t$ calibrates the uncertainty bonus so that the
resulting UCB is an upper bound on true model quality on this
high-probability event.
For a fixed experimental horizon $T$, we may use a constant $\beta_t=\beta$ for all $t\leq T$. The theoretical guarantees remain valid provided that $\beta$ upper-bounds the confidence radii in~\eqref{eq:confidence-radius} over this horizon. 
Choosing a larger constant $\beta$ preserves the confidence guarantee, although it may encourage excessive exploration and increase the regret bound. As verified observations accumulate, uncertainty decreases in the feature directions they cover, reducing the corresponding exploration bonuses.

% Order the selected verified features within each query before its labels are observed, and concatenate these orders across queries. Assumption~\ref{ass:quality} requires each label error to have zero conditional mean and conditional moment-generating function at most $\exp(s^2R^2/2)$ for every $s\in\mathbb R$, given its selected feature and all preceding observations, including earlier labels in the same query. This martingale-noise condition permits adaptive feature selection and is the condition used below for self-normalized concentration. It requires unbiased feedback for true expected correctness; fixing a verifier alone does not ensure this property.

\textbf{Sublinear regret in quality learning.}
We now prove Theorem~\ref{thm:sublinear-regret}.
We first bound the cumulative uncertainty of the verified
task-model pairs. Notice that for each query, its selected nodes for verification are decided all at once by solving the MILP as described in JOVE.
Let $\mathcal V_t=\{i\in\mathcal N_t:h_t^i=1\}$ denote the
set of nodes selected for verification, and write
$\phi_{t,i}=\phi_t^i(a_t^i)$ for convenience. 
From round $t-1$ to $t$, the information matrix updates as
\[
G_t
=
G_{t-1}
+\sum_{i\in\mathcal V_t}\phi_{t,i}\phi_{t,i}^{\top},
\qquad
G_0=\lambda_{\mathrm{reg}}I.
\]
Recall that
\[
[u_t^i(a_t^i)]^2
=
\phi_{t,i}^{\top}G_{t-1}^{-1}\phi_{t,i}.
\]
Define
\[
N:=Tn_{\max},
\qquad
C_{\mathrm{bat}}
:=
\left(
1+\frac{(n_{\max}-1)L^2}{\lambda_{\mathrm{reg}}}
\right)
\left(
1+\frac{L^2}{\lambda_{\mathrm{reg}}}
\right),
\]
and
\[
\Gamma_T
:=
C_{\mathrm{bat}}\,d
\log\left(
1+\frac{NL^2}{d\lambda_{\mathrm{reg}}}
\right).
\]
Here, $N$ upper-bounds the total number of nodes over
$T$ queries, and $C_{\mathrm{bat}}$ accounts for the use
of a shared information matrix within each query.

\begin{lemma}[Batched uncertainty bound]
\label{lem:batch-potential}
Under Assumptions~\ref{ass:quality}--\ref{ass:regret},
\[
\sum_{t=1}^T\sum_{i\in\mathcal N_t}
h_t^i[u_t^i(a_t^i)]^2
\leq \Gamma_T,
\qquad
\sum_{t=1}^T\sum_{i\in\mathcal N_t}
h_t^i I_t^i(a_t^i)
\leq \frac{\Gamma_T}{2}.
\]
\end{lemma}

\begin{proof}
\textit{Step 1: Compare batch and sequential uncertainties.}
For each query $t$, order its verified features as
$\phi_{t,1},\ldots,\phi_{t,b_t}$, where
$b_t=|\mathcal V_t|\leq n_{\max}$.
Define the intermediate matrices
\[
V_{t,0}=G_{t-1},
\qquad
V_{t,j}
=
V_{t,j-1}+\phi_{t,j}\phi_{t,j}^{\top},
\quad j=1,\ldots,b_t.
\]
Thus, $V_{t,b_t}=G_t$.
These matrices are introduced only for the analysis;
the algorithm computes every uncertainty in query $t$
using $G_{t-1}$.

Since $\|\phi_{t,k}\|_2\leq L$, we have
$\phi_{t,k}\phi_{t,k}^{\top}\preceq L^2I$.
Also, $G_{t-1}\succeq\lambda_{\mathrm{reg}}I$ as it only adds positive semi-definite matrices per round, so
$I\preceq G_{t-1}/\lambda_{\mathrm{reg}}$.
Consequently, for each $j=1,\ldots,b_t$,
\[
\begin{aligned}
V_{t,j-1}
&=
G_{t-1}
+\sum_{k=1}^{j-1}\phi_{t,k}\phi_{t,k}^{\top}\\
&\preceq
G_{t-1}+(j-1)L^2I\\
&\preceq
\left(
1+\frac{(j-1)L^2}{\lambda_{\mathrm{reg}}}
\right)G_{t-1}\\
&\preceq
\underbrace{
\left(
1+\frac{(n_{\max}-1)L^2}{\lambda_{\mathrm{reg}}}
\right)
}_{=:\,\alpha_{\mathrm{bat}}}
G_{t-1}.
\end{aligned}
\]
Inversion reverses the positive-definite matrix order.
Therefore,
\[
V_{t,j-1}^{-1}
\succeq
\alpha_{\mathrm{bat}}^{-1}G_{t-1}^{-1},
\qquad\text{or equivalently}\qquad
G_{t-1}^{-1}
\preceq
\alpha_{\mathrm{bat}}V_{t,j-1}^{-1}.
\]
Writing
\[
q_{t,j}
:=
\phi_{t,j}^{\top}V_{t,j-1}^{-1}\phi_{t,j},
\]
By using the above inequality, we obtain
\[
\phi_{t,j}^{\top}G_{t-1}^{-1}\phi_{t,j}
\leq
\alpha_{\mathrm{bat}}q_{t,j}.
\]
Summing over all verified nodes gives
\[
\sum_{t=1}^T\sum_{i\in\mathcal N_t}
h_t^i[u_t^i(a_t^i)]^2
\leq
\alpha_{\mathrm{bat}}
\sum_{t=1}^T\sum_{j=1}^{b_t}q_{t,j}.
\]

\textit{Step 2: Bound sequential uncertainties by
log-determinant increments.}
Because $V_{t,j-1}\succeq\lambda_{\mathrm{reg}}I$,
\[
0\leq q_{t,j}
\leq
\frac{\|\phi_{t,j}\|_2^2}{\lambda_{\mathrm{reg}}}
\leq
\frac{L^2}{\lambda_{\mathrm{reg}}}
=:r.
\]
Notice that for any $x\in[0,r]$,
\[
\log(1+x)
=
\int_0^x\frac{1}{1+s}\,ds
\geq
\frac{x}{1+r}.
\]
Hence $q_{t,j}\leq(1+r)\log(1+q_{t,j})$.
The matrix determinant lemma gives
\[
\begin{aligned}
\det(V_{t,j})
&=
\det(V_{t,j-1})
\left(
1+\phi_{t,j}^{\top}
V_{t,j-1}^{-1}\phi_{t,j}
\right)\\
&=
\det(V_{t,j-1})(1+q_{t,j}),
\end{aligned}
\]
and therefore
\[
\log(1+q_{t,j})
=
\log\det(V_{t,j})-\log\det(V_{t,j-1}).
\]
Summing these increments first within each query and
then across queries yields
\[
\begin{aligned}
\sum_{t=1}^T\sum_{j=1}^{b_t}\log(1+q_{t,j})
&=
\sum_{t=1}^T
\left[
\log\det(V_{t,b_t})-\log\det(V_{t,0})
\right]\\
&=
\sum_{t=1}^T
\left[
\log\det(G_t)-\log\det(G_{t-1})
\right]\\
&=
\log\frac{\det(G_T)}{\det(G_0)}.
\end{aligned}
\]
Queries with $b_t=0$ contribute zero to both sides.
Combining this identity with Step~1, and using
$C_{\mathrm{bat}}=\alpha_{\mathrm{bat}}(1+r)$, gives
\[
\sum_{t=1}^T\sum_{i\in\mathcal N_t}
h_t^i[u_t^i(a_t^i)]^2
\leq
C_{\mathrm{bat}}
\log\frac{\det(G_T)}{\det(G_0)}.
\]

\textit{Step 3: Bound the determinant ratio.}
The total number of verified observations satisfies
$\sum_{t=1}^T b_t\leq Tn_{\max}=N$.
Using $\operatorname{tr}(\phi\phi^\top)=\|\phi\|_2^2$,
we obtain
\[
\begin{aligned}
\operatorname{tr}(G_T)
&=
d\lambda_{\mathrm{reg}}
+\sum_{t=1}^T\sum_{j=1}^{b_t}
\|\phi_{t,j}\|_2^2\\
&\leq
d\lambda_{\mathrm{reg}}+NL^2.
\end{aligned}
\]
Let $\mu_1,\ldots,\mu_d>0$ be the eigenvalues of $G_T$.
The arithmetic--geometric mean inequality implies
\[
\det(G_T)
=
\prod_{\ell=1}^d\mu_\ell
\leq
\left(\frac{\sum_{\ell=1}^d\mu_\ell}{d}\right)^d
=
\left(\frac{\operatorname{tr}(G_T)}{d}\right)^d
\leq
\left(
\lambda_{\mathrm{reg}}+\frac{NL^2}{d}
\right)^d.
\]
Since $\det(G_0)=\lambda_{\mathrm{reg}}^d$,
\[
\log\frac{\det(G_T)}{\det(G_0)}
\leq
d\log\left(
1+\frac{NL^2}{d\lambda_{\mathrm{reg}}}
\right).
\]
Substituting this into the bound from Step~2 proves
\[
\sum_{t=1}^T\sum_{i\in\mathcal N_t}
h_t^i[u_t^i(a_t^i)]^2
\leq \Gamma_T.
\]

\textit{Step 4: Bound the cumulative information gain.}
By the definition of $I_t^i(a)$ and the inequality
$\log(1+x)\leq x$ for $x\geq0$,
\[
I_t^i(a)
=
\frac12\log\left(1+[u_t^i(a)]^2\right)
\leq
\frac12[u_t^i(a)]^2.
\]
Multiplying by $h_t^i\geq0$ and summing gives
\[
\sum_{t=1}^T\sum_{i\in\mathcal N_t}
h_t^i I_t^i(a_t^i)
\leq
\frac12
\sum_{t=1}^T\sum_{i\in\mathcal N_t}
h_t^i[u_t^i(a_t^i)]^2
\leq
\frac{\Gamma_T}{2},
\]
which proves the second claim.
\end{proof}

Here we provide a lemma to convert the lower bound of $N=Tn_{\mathrm{max}}$ to the lower bound of the horizon $T$. 
\begin{lemma}[Sufficient horizon length]\label{lem:burn-in}
Let $N=Tn_{\max}$ and recall that
\[
C_{\mathrm{bat}}
=
\left(1+\frac{(n_{\max}-1)L^2}{\lambda_{\mathrm{reg}}}\right)
\left(1+\frac{L^2}{\lambda_{\mathrm{reg}}}\right),
\qquad
\Gamma_T
=
C_{\mathrm{bat}}d
\log\left(1+\frac{NL^2}{d\lambda_{\mathrm{reg}}}\right).
\]
Under $\lambda_{\mathrm{reg}}\geq L^2$, we have
$1\leq C_{\mathrm{bat}}\leq2n_{\max}$.
Furthermore, $\Gamma_T\leq N$ whenever
\begin{equation}\label{eq:burn-in}
N\geq4C_{\mathrm{bat}}d\log(4C_{\mathrm{bat}}).
\end{equation}
In particular, a sufficient condition is
\[
T\geq8d\log(8n_{\max}).
\]
\end{lemma}

\begin{proof}
\textit{Step 1: Bound the batching constant.}
Write
\[
r:=\frac{L^2}{\lambda_{\mathrm{reg}}}\in[0,1],
\qquad
C:=C_{\mathrm{bat}}.
\]
Since $n_{\max}\geq1$,
\[
1\leq1+(n_{\max}-1)r\leq n_{\max},
\qquad
1\leq1+r\leq2.
\]
Multiplying these inequalities yields
\[
1\leq C=(1+(n_{\max}-1)r)(1+r)\leq2n_{\max}.
\]

\textit{Step 2: Establish $\Gamma_T\leq N$.}
Set $x:=N/d$. Then
\[
\Gamma_T=Cd\log(1+rx),
\qquad
N=dx.
\]
Because $r\leq1$ and the logarithm is increasing,
it suffices to show that $C\log(1+x)\leq x$.

By concavity, $\log y$ lies below its tangent at
$y=2C$. Thus, for every $y>0$,
\[
\log y
\leq
\log(2C)+\frac{y-2C}{2C}.
\]
Taking $y=1+x$ and using $C\geq1$ gives
\[
\begin{aligned}
\log(1+x)
&\leq
\log(2C)+\frac{1+x-2C}{2C}\\
&=
\log(2C)+\frac{x}{2C}
+\frac{1-2C}{2C}\\
&\leq
\log(2C)+\frac{x}{2C}.
\end{aligned}
\]
Condition~\eqref{eq:burn-in} is equivalent to
$x\geq4C\log(4C)$. Under this condition,
\[
\frac{x}{2C}
\geq2\log(4C)
\geq\log(2C),
\]
where the last inequality holds because $C\geq1$.
Consequently,
\[
\log(1+x)
\leq
\frac{x}{2C}+\frac{x}{2C}
=
\frac{x}{C}.
\]
Substituting into the expression for $\Gamma_T$ yields
\[
\Gamma_T
=
Cd\log(1+rx)
\leq
Cd\log(1+x)
\leq
Cd\frac{x}{C}
=
dx
=
N.
\]

\textit{Step 3: Derive a condition directly on $T$.}
The function $f(c)=c\log(4c)$ is increasing for
$c\geq1$, since
\[
f'(c)=\log(4c)+1>0.
\]
Using $1\leq C\leq2n_{\max}$, we therefore obtain
\[
4Cd\log(4C)
\leq
4(2n_{\max})d\log(8n_{\max})
=
8n_{\max}d\log(8n_{\max}).
\]
If $T\geq8d\log(8n_{\max})$, then
\[
N=n_{\max}T
\geq
8n_{\max}d\log(8n_{\max})
\geq
4Cd\log(4C).
\]
Thus, condition~\eqref{eq:burn-in} holds, and Step~2
implies $\Gamma_T\leq N$.
\end{proof}

\textbf{Proof of Theorem~\ref{thm:sublinear-regret}}
\begin{proof}
Write $n_t:=|\mathcal N_t|$ and $N:=Tn_{\max}$.
Then $\sum_{t=1}^T n_t\leq N$.
By the theorem statement, $c_v>0$. Define
\(
\beta_T^{\max}:=\max_{1\leq t\leq T}\beta_t.
\)
Use the confidence radii in~\eqref{eq:confidence-radius}.
Order the verified observations within each query before
observing their labels, and concatenate these orders
across queries. The selected features and verification
decisions are determined before observing the corresponding
labels. Under Assumption~\ref{ass:quality}, each label error
is zero-mean and $R$-sub-Gaussian given
the selected features and all preceding observations,
including earlier labels in the same query.
Thus, the confidence-ellipsoid result of
\citet[Theorem~2]{abbasi2011improved}, together with
the determinant--trace bound, gives
\[
\mathbb P(\mathcal E)\geq1-\delta_{\mathrm{conf}},
\qquad
\mathcal E
:=
\left\{
\|\hat\theta_t-\theta^\star\|_{G_{t-1}}
\leq\beta_t
\text{ for every }t\leq T
\right\}.
\]

For every candidate task--model pair, Cauchy--Schwarz gives
\[
\begin{aligned}
\left|
\phi_t^i(a)^\top\hat\theta_t-\bar y_t^i(a)
\right|
&=
\left|
\phi_t^i(a)^\top(\hat\theta_t-\theta^\star)
\right|\\
&\leq
\|\phi_t^i(a)\|_{G_{t-1}^{-1}}
\|\hat\theta_t-\theta^\star\|_{G_{t-1}}\\
&\leq
\beta_tu_t^i(a).
\end{aligned}
\]
Consequently, the untruncated optimistic estimate
\[
s_t^i(a)
:=
\phi_t^i(a)^\top\hat\theta_t+\beta_tu_t^i(a)
\]
satisfies
\[
\bar y_t^i(a)
\leq s_t^i(a)
\leq \bar y_t^i(a)+2\beta_tu_t^i(a).
\]
Since $\bar y_t^i(a)\in[0,1]$, clipping $s_t^i(a)$
to $[0,1]$ preserves its lower bound by $\bar y_t^i(a)$
and cannot increase its excess over $\bar y_t^i(a)$.
Hence
\begin{equation}\label{eq:optimism-event}
0
\leq
\mathrm{UCB}_t^i(a)-\bar y_t^i(a)
\leq
2\beta_tu_t^i(a)
\end{equation}
simultaneously for all rounds and candidate pairs.

The assumed horizon condition and
Lemma~\ref{lem:burn-in} imply $\Gamma_T\leq N$.
For $\Gamma_T>0$, define
\[
\tau_T:=\left(\frac{\Gamma_T}{N}\right)^{1/3}
\in(0,1],
\qquad
K_T:=\frac{2n_{\max}c_v}{\log(1+\tau_T^2)}.
\]
The choice of $\tau_T$ gives the identity
\(
\Gamma_T=N\tau_T^3,
\)
which will be used to simplify the cumulative regret bound.
Consider any selected task--model pair satisfying
$u_t^i(a_t^i)>\tau_T$.
By the definition of the information gain,
\[
I_t^i(a_t^i)
=
\frac12\log\left(1+[u_t^i(a_t^i)]^2\right)
>
\frac12\log(1+\tau_T^2).
\]
Assumption~\ref{ass:regret} also gives
\[
0\leq
\lambda_t\hat{\widetilde m}_t^i
\leq
\lambda_{\max}\widetilde m_{\max}
=
c_v.
\]
Since $k_v\geq K_T>0$, we obtain
\[
\begin{aligned}
\frac{k_v}{n_t}I_t^i(a_t^i)
-\lambda_t\hat{\widetilde m}_t^i
&\geq
\frac{K_T}{n_t}I_t^i(a_t^i)-c_v\\
&\geq
\frac{K_T}{n_{\max}}I_t^i(a_t^i)-c_v\\
&>
\frac{K_T}{2n_{\max}}\log(1+\tau_T^2)-c_v\\
&=0.
\end{aligned}
\]
For a fixed executor choice, changing $h_t^i$ from
$0$ to $1$ changes the online objective by precisely
this quantity, because the verification bonus in
~\eqref{eq:online-milp} is normalized by $n_t$, while
the priced verification cost is not. This change leaves
feasibility unchanged.
Therefore, if $h_t^i=0$, this change would strictly
improve the objective, contradicting exact optimality.
We conclude that
\begin{equation}\label{eq:forced-verification}
u_t^i(a_t^i)>\tau_T
\quad\Longrightarrow\quad
h_t^i=1.
\end{equation}

For a feasible allocation $\mathbf z$, define
\[
\widehat J_t(\mathbf z)
:=
\sum_{i\in\mathcal N_t}
\sum_{a\in\mathcal A_t^i}
\sum_{\eta\in\{0,1\}}
\left[
\frac{1}{n_t}\mathrm{UCB}_t^i(a)
-\lambda_t\hat c_t^i(a,\eta)
\right]z^i(a,\eta),
\]
and
\[
B_t(\mathbf z)
:=
\frac{1}{n_t}
\sum_{i\in\mathcal N_t}
\sum_{a\in\mathcal A_t^i}
\sum_{\eta\in\{0,1\}}
\eta I_t^i(a)z^i(a,\eta).
\]
By ~\eqref{eq:online-milp}, JOVE maximizes
$\widehat J_t(\mathbf z)+k_vB_t(\mathbf z)$
over the same feasible set used by the comparator.

On $\mathcal E$, optimism and the nonnegativity of
the allocation variables imply
\[
J_t(\mathbf z_t^\star)
\leq
\widehat J_t(\mathbf z_t^\star).
\]
Exact optimality of JOVE gives
\[
\widehat J_t(\mathbf z_t)+k_vB_t(\mathbf z_t)
\geq
\widehat J_t(\mathbf z_t^\star)
+k_vB_t(\mathbf z_t^\star).
\]
Rearranging this inequality and using optimism yields
\[
\begin{aligned}
J_t(\mathbf z_t^\star)-J_t(\mathbf z_t)
&\leq
\widehat J_t(\mathbf z_t^\star)-J_t(\mathbf z_t)\\
&\leq
\widehat J_t(\mathbf z_t)-J_t(\mathbf z_t)
+k_v\left[
B_t(\mathbf z_t)-B_t(\mathbf z_t^\star)
\right]\\
&\leq
\widehat J_t(\mathbf z_t)-J_t(\mathbf z_t)
+k_vB_t(\mathbf z_t),
\end{aligned}
\]
where the last step uses $k_v\geq0$ and
$B_t(\mathbf z_t^\star)\geq0$.

At the same allocation $\mathbf z_t$, the resource-cost
terms in $\widehat J_t$ and $J_t$ are identical and cancel.
Since exactly one executor--verification action is
selected for each node,
\[
\begin{aligned}
\widehat J_t(\mathbf z_t)-J_t(\mathbf z_t)
=
\frac{1}{n_t}
\sum_{i\in\mathcal N_t}
\left[
\mathrm{UCB}_t^i(a_t^i)-\bar y_t^i(a_t^i)
\right]\leq
\frac{2\beta_t}{n_t}
\sum_{i\in\mathcal N_t}u_t^i(a_t^i),
\end{aligned}
\]
and
\[
B_t(\mathbf z_t)
=
\frac{1}{n_t}
\sum_{i\in\mathcal N_t}
h_t^iI_t^i(a_t^i).
\]
Combining these expressions proves
\begin{equation}\label{eq:regret-one-query}
\begin{aligned}
J_t(\mathbf z_t^\star)-J_t(\mathbf z_t)
\leq{}
\frac{2\beta_t}{n_t}
\sum_{i\in\mathcal N_t}u_t^i(a_t^i)+
\frac{k_v}{n_t}
\sum_{i\in\mathcal N_t}h_t^iI_t^i(a_t^i).
\end{aligned}
\end{equation}
Only the resource terms evaluated at the same allocation
cancel; no equality of costs between JOVE and the
comparator is required.

Partition the executed nodes according to whether their
uncertainty exceeds $\tau_T$:
\[
\mathcal N_t^{>}
:=
\{i\in\mathcal N_t:u_t^i(a_t^i)>\tau_T\},
\qquad
\mathcal N_t^{\leq}
:=
\{i\in\mathcal N_t:u_t^i(a_t^i)\leq\tau_T\}.
\]
For the nodes in $\mathcal N_t^{\leq}$, each uncertainty
is at most $\tau_T$. Since there are at most $N$
executed nodes in total,
\[
\sum_{t=1}^T\sum_{i\in\mathcal N_t^{\leq}}u_t^i(a_t^i)
\leq
\tau_T\sum_{t=1}^T|\mathcal N_t^{\leq}|
\leq N\tau_T.
\]
For every node in $\mathcal N_t^{>}$,
~\eqref{eq:forced-verification} guarantees
$h_t^i=1$. Moreover, $u_t^i(a_t^i)>\tau_T$ implies
\[
u_t^i(a_t^i)
\leq
\frac{[u_t^i(a_t^i)]^2}{\tau_T}.
\]
Therefore, Lemma~\ref{lem:batch-potential} gives
\[
\begin{aligned}
\sum_{t=1}^T\sum_{i\in\mathcal N_t^{>}}u_t^i(a_t^i)
&\leq
\frac{1}{\tau_T}
\sum_{t=1}^T\sum_{i\in\mathcal N_t^{>}}
[u_t^i(a_t^i)]^2=
\frac{1}{\tau_T}
\sum_{t=1}^T\sum_{i\in\mathcal N_t^{>}}
h_t^i[u_t^i(a_t^i)]^2\\
&\leq
\frac{1}{\tau_T}
\sum_{t=1}^T\sum_{i\in\mathcal N_t}
h_t^i[u_t^i(a_t^i)]^2\\
&\leq
\frac{\Gamma_T}{\tau_T}.
\end{aligned}
\]
Adding the two contributions yields
\begin{equation}\label{eq:cumulative-uncertainty}
\sum_{t=1}^T\sum_{i\in\mathcal N_t}u_t^i(a_t^i)
\leq
N\tau_T+\frac{\Gamma_T}{\tau_T}.
\end{equation}

Sum~\eqref{eq:regret-one-query} over $t$.
Using $n_t\geq1$, $\beta_t\leq\beta_T^{\max}$,
and the nonnegativity of uncertainty and information gain,
we obtain
\[
\begin{aligned}
\mathcal R_T
\leq
2\beta_T^{\max}
\sum_{t=1}^T\sum_{i\in\mathcal N_t}u_t^i(a_t^i)
+
k_v\sum_{t=1}^T\sum_{i\in\mathcal N_t}
h_t^iI_t^i(a_t^i)\leq
2\beta_T^{\max}
\left(
N\tau_T+\frac{\Gamma_T}{\tau_T}
\right)
+\frac{k_v\Gamma_T}{2},
\end{aligned}
\]
where the last inequality uses
\eqref{eq:cumulative-uncertainty} and the second bound
in Lemma~\ref{lem:batch-potential}.

We now bound the two terms separately.
Because $\Gamma_T=N\tau_T^3$ and $\tau_T\leq1$,
\[
N\tau_T+\frac{\Gamma_T}{\tau_T}
=
N\tau_T+N\tau_T^2
\leq
2N\tau_T.
\]
Hence
\[
2\beta_T^{\max}
\left(
N\tau_T+\frac{\Gamma_T}{\tau_T}
\right)
\leq
4\beta_T^{\max}N\tau_T.
\]

For the information-gain term, note that
$0\leq\tau_T^2\leq1$ implies
\[
\log(1+\tau_T^2)
=
\int_0^{\tau_T^2}\frac{1}{1+s}\,ds
\geq
\frac{\tau_T^2}{2}.
\]
Using the upper endpoint $k_v\leq2K_T$ gives
\[
\begin{aligned}
\frac{k_v\Gamma_T}{2}
&\leq K_T\Gamma_T=
\frac{2n_{\max}c_v\Gamma_T}{\log(1+\tau_T^2)}\leq
\frac{4n_{\max}c_v\Gamma_T}{\tau_T^2}=
4n_{\max}c_vN\tau_T.
\end{aligned}
\]
Combining the two bounds and substituting the definition
of $\tau_T$, we conclude that
\[
\begin{aligned}
\mathcal R_T
&\leq
4(n_{\max}c_v+\beta_T^{\max})N\tau_T\\
&=
4(n_{\max}c_v+\beta_T^{\max})N^{2/3}\Gamma_T^{1/3}\\
&=
4\left(n_{\max}c_v+\max_{t\leq T}\beta_t\right)
(n_{\max}T)^{2/3}\Gamma_T^{1/3}.
\end{aligned}
\]
This inequality holds on $\mathcal E$, whose probability
is at least $1-\delta_{\mathrm{conf}}$.

Since $|\mathcal D_{t-1}|\leq N$, the confidence radii
satisfy the deterministic upper bound
\[
\beta_T^{\max}
\leq
\sqrt{\lambda_{\mathrm{reg}}}S
+
R\sqrt{
2\log(1/\delta_{\mathrm{conf}})
+
d\log\left(
1+\frac{NL^2}{d\lambda_{\mathrm{reg}}}
\right)}.
\]
For fixed $n_{\max}$, $L$, $S$, $R$,
$\lambda_{\mathrm{reg}}$, $\lambda_{\max}$,
$\widetilde m_{\max}$, and $\delta_{\mathrm{conf}}$,
we have
\[
\Gamma_T
=
O\!\left(d\log(1+T)\right),
\qquad
n_{\max}c_v+\beta_T^{\max}
=
O\!\left(\sqrt{d\log(1+T)}\right).
\]
Substituting these estimates into the regret bound gives
\[
\begin{aligned}
\mathcal R_T
&=
O\!\left(
[d\log(1+T)]^{1/2}
T^{2/3}
[d\log(1+T)]^{1/3}
\right)\\
&=
O\!\left(
d^{5/6}T^{2/3}[\log(1+T)]^{5/6}
\right).
\end{aligned}
\]
In particular, for fixed $d$,
\[
\frac{\mathcal R_T}{T}
=
O\!\left(
T^{-1/3}[\log(1+T)]^{5/6}
\right)
\longrightarrow0.
\]
\end{proof}

The lower endpoint ensures that sufficiently uncertain
selections are verified, while the upper endpoint
controls the cumulative information-gain bonus.
The factor two is not essential: for any fixed $\rho\geq1$,
choosing $k_v\in[K_T,\rho K_T]$ gives
\(
\mathcal R_T
\leq
\left(4\beta_T^{\max}+2\rho n_{\max}c_v\right)
N^{2/3}\Gamma_T^{1/3}
\)
by the same argument.
% The regret bound compares quality objectives on JOVE's
% estimated feasible sets using the same adaptive budget
% prices. We do not establish optimal tuning of $k_v$,
% calibration of resource estimates, regret for final-answer
% accuracy, or minimax optimality.
The horizon condition $T\geq 8d\log(8n_\mathrm{max})$ is sufficient but need
not be necessary: it is used to ensure
$\Gamma_T\leq N$ in the proof steps.
This condition implies that lower $d$ (the dimension of the feature vector) yields a smaller condition on $T$, potentially improving the learning speed. Also, we point out that a small $d$ may restrict the expressivity of the feature vector, which may potentially degrade the model quality estimation.

\paragraph{Theory--practice scope.}
Theorem~\ref{thm:sublinear-regret} is intended as an existence result showing that a nonempty range of verification weights ($k_v$) yields sublinear quality-learning regret under the stated assumptions, rather than as a practical tuning rule. Its horizon condition and weight interval are sufficient worst-case conditions arising from conservative concentration and batching bounds. For our 384-dimensional features and $n_{\max}=5$, the sufficient horizon can exceed the lengths of some experimental sequences. In practice, substantially less exploration may be sufficient. The experiments evaluate JOVE outside the theorem's conservative sufficient regime, while the theorem establishes the existence of verification weights that guarantee sublinear regret under the formal model.

% NEW: explicitly incomplete implementation record, not fabricated procedures.
\subsection{Implementation details}
\label{app:implementation}
\subsubsection{JOVE procedure}
\label{app:algorithm}
\begin{algorithm}[h]
\caption{JOVE}
\label{alg:jove}
\begin{algorithmic}[1]
\Require Planner, fixed verifier, budget $\gamma$, step size $\alpha_\lambda$, regularization $\lambda_{\mathrm{reg}}$, confidence radii $\beta_t$, and verification weight $k_v$. \State Initialize $G_0=\lambda_{\mathrm{reg}}I$, $\hat\theta_1=\mathbf 0$, $\lambda_1=0$, and resource estimates.
\For{each query $x_t$}
    \State Generate DAG $\mathcal G_t$; record planning cost and reserve allocation time.
    \State Use embedding model to form context features $\phi_t^i(a)$ for each node-model pair.
    \State Compute $\mathrm{UCB}_t^i(a)$,
    $I_t^i(a)$, $\hat c_t^i(a,\eta)$, and
    $\hat d_t^i(a)$ using prior observations.
    \State Solve MILP (\eqref{eq:online-milp}) to decide
    $\mathbf z_t$ variables, yielding $(a_t^i,h_t^i)$ for every node.
    \State Execute the DAG following dependencies, launching selected verifications as node outputs become available.
    \State Deliver the final response without waiting for verification.
\State Collect verification labels and update the quality model
via~\eqref{eq:ridge}.
    \State Update resource histories using observed
    execution costs, latencies, and verification costs.
    \State Compute total realized cost
    $c_t=m_t^{\mathrm{plan}}
    +\sum_{i\in\mathcal N_t}m_t^i(a_t^i)
    +\sum_{i:h_t^i=1}\widetilde m_t^i$.
    \State Update
    $\lambda_{t+1}=
    [\lambda_t+\alpha_\lambda(c_t-\gamma)]_+$.
\EndFor
\end{algorithmic}
\end{algorithm}
\newpage
\subsubsection{Planner and verification prompts}
\label{app:prompt}

\begin{lstlisting}[
    basicstyle=\small\rmfamily,
    breaklines=true,
    breakatwhitespace=true,
    breakindent=0pt,
    columns=fullflexible,
    keepspaces=true,
    showstringspaces=false,
    backgroundcolor=\color{gray!5},
    rulecolor=\color{gray!30},
    frame=single,
    frameround=tttt,
    framesep=7pt,
    framerule=0.5pt,
    aboveskip=10pt,
    belowskip=10pt,
    title={\small\bfseries Planner Prompt}
]
You are a planner for benchmark questions. Decompose the user prompt into {node_count} small executable tasks arranged as a DAG. Return JSON only; no markdown or extra text. Do not solve the question, compute the final answer, or include reasoning. For multi-hop factual QA, split anchor/entity lookup from the final attribute lookup. For math, science, multiple-choice, or reasoning questions, split solution reasoning from final answer extraction. Each task needs lowercase fields "id", "description", "predecessors", "type", "output_format", "input_template". "type" must be one of: {task_groups_list}. Do not create task nodes whose purpose is to verify, check, validate, confirm, or audit another task output. Use precise output_format values such as "full name only", "date only", "year only", "number only", "integer only", "option letter only", or "short phrase only". The sink task output_format must satisfy any final-format instruction in the user message. The sink input_template must ask for the original final answer using predecessors. Keep dependencies minimal and acyclic.
Schema: {"tasks":[{"id":"t1","description":"...","predecessors":[],"type":"Anchor Identification","output_format":"full name only","input_template":"..."}]}.
Question: {...}
\end{lstlisting}

\begin{lstlisting}[
    basicstyle=\small\rmfamily,
    breaklines=true,
    breakatwhitespace=true,
    breakindent=0pt,
    columns=fullflexible,
    keepspaces=true,
    showstringspaces=false,
    backgroundcolor=\color{gray!5},
    rulecolor=\color{gray!30},
    frame=single,
    frameround=tttt,
    framesep=7pt,
    framerule=0.5pt,
    aboveskip=10pt,
    belowskip=10pt,
    title={\small\bfseries Verifier Prompt}]
Original user prompt: {prompt}. Task to verify ({task.id}): {instruction}. Output format requirement: {task.output_format}. Executor output: {executor_output}. 
Is the executor output factually correct for this task?
decision=true only if the required result is right given the original prompt and dependencies.
decision=false if a name/date/number/place/option is wrong or invented, the answer is a fluent guess you cannot verify, the output is empty/refusal/off-task, or only the format is right.
If you are unsure, prefer decision=false.
Return JSON only: {"reason": "short reason", "decision": true/false}.
\end{lstlisting}

\subsubsection{Embedding model}
\label{app:embedding}
% For each task-model pair, we embed the task instruction, descriptions of required node outputs, and candidate model description into a feature vector $\phi_t^i(a)\in\mathbb R^d$
We use a frozen text embedding as the feature map for the LinUCB head. Specifically, to construct $\phi_t^i(a)$, we encode the node's description, descriptions of required node outputs, and LLM API description with Sentence-Transformers all-MiniLM-L6-v2~\citep{reimers2019sentence}, which produces a 384-dimensional unit-normalized feature vector. We choose this model because it is small ($\sim$22M parameters) and runs fast even on CPU, reducing the reserved time for executor allocation. 
Larger modern embedding models (e.g., the Qwen3 Embedding family, such as the 0.6B or 4B variants with up to 1024- or 2560-dimensional outputs) may improve retrieval-oriented embedding quality, but they are substantially slower and more memory-intensive to evaluate for every task–model pair, increase the LinUCB feature dimension (and thus the cost of maintaining and inverting the covariance matrix). Moreover, Theorem~\ref{thm:sublinear-regret} gives a larger sufficient horizon and a looser regret bound as $d$ increases, suggesting potentially slower learning. These considerations motivate our choice of a compact encoder.

\subsubsection{Resource estimators}
\label{app:estimators}

JOVE's MILP uses online estimates of expected execution and verification costs, together with executor latency quantiles. For each executor API~$a$, we maintain a bounded history of recent calls,
\[
\Upsilon_{t-1}(a)
\subseteq
\left\{
(\phi_s^j(a),m_s^j(a),d_s^j(a)):
s<t,\ j\in\mathcal N_s,\ a_s^j=a
\right\}.
\]
$\Upsilon_{t-1}(a)$ contains the execution records for model $a$ collected before query $t$. Each record $(\phi,m,d)$ contains the feature vector, realized API cost, and observed latency.
We estimate expected cost and latency quantiles by weighting historical observations according to their feature similarity to the current task.
For a nonempty history, define
\begin{equation}\label{eq:empirical-resource-estimates}
\hat m_{t,\mathrm{emp}}^i(a)
=\!\!\!\!\!\!\!\!\!\!\!
\sum_{(\phi,m,d)\in\Upsilon_{t-1}(a)}\!\!\!\!\!\!\!\!\!\!\! w\left(\phi,\phi_t^i(a)\right) m
\quad \text{with}\quad
w\left(\phi,\phi_t^i(a)\right)
=
\frac{\exp\!\left(\phi_t^i(a)^\top\phi\right)}
{\displaystyle
\!\!\!\!\!\!\!\!\!\!\!\sum_{(\phi',m',d')\in\Upsilon_{t-1}(a)}\!\!\!\!\!\!\!\!\!\!\!
 \exp\!\left(\phi_t^i(a)^\top\phi'\right)}.
\end{equation}
For unit-normalized vectors, the inner product in \eqref{eq:empirical-resource-estimates} is cosine similarity. Thus, $\hat m_{t,\mathrm{emp}}^i(a)$ is a weighted average of the observed costs, with greater weight assigned to calls with features similar to the current task-model pair. Using the same weights, we estimate the latency quantile as
\begin{equation}\label{eq:weighted-latency-quantile}
\hat \ell_{t,\mathrm{emp}}^i(a,\epsilon)
=
\inf\!\big\{
u\geq0:\!\!\!\!\!\!\!\!\!\!\!
\sum_{(\phi,m,d)\in\Upsilon_{t-1}(a)}\!\!\!\!\!\!\!\!\!\!\!
w\!\left(\phi,\phi_t^i(a)\right)\mathbf 1\{d\leq u\}
\geq1-\epsilon
\big\}.
\end{equation}
This is the smallest latency threshold containing at least $1-\epsilon$ of the total observation weight. 
% To stabilize estimates with scarce observations, we initialize them from predefined model-specific baselines and gradually reduce the baseline weight as observations accumulate, approaching the estimates in \eqref{eq:empirical-resource-estimates} and \ref{eq:weighted-latency-quantile}.

To handle sparse or empty histories, we initialize each model with a baseline cost $m_0(a)$ and the baseline latency quantile $\ell_0(a,\delta_t/|\mathcal{N}_t|)$.
These baselines can be informed by platform statistics or at least one preliminary profiling call for each model. 

Let $n_{t-1}(a)=|\Upsilon_{t-1}(a)|$ and define
the baseline weight
\begin{equation}
\rho_{t-1}(a)=\frac{1}{1+n_{t-1}(a)}.
\label{eq:prior-weight}
\end{equation}
For nonempty histories,
$\hat m_{t,\mathrm{emp}}^i(a)$ and
$\hat\ell_{t,\mathrm{emp}}^i(a,\epsilon)$ denote
the feature-weighted empirical estimates, respectively.
The estimates used by JOVE are
\begin{align}
\hat m_t^i(a)
&=
\rho_{t-1}(a)m_0(a)
+
\bigl(1-\rho_{t-1}(a)\bigr)
\hat m_{t,\mathrm{emp}}^i(a),
\label{eq:shrunk-cost}
\\
\hat\ell_t^i(a,\epsilon)
&=
\rho_{t-1}(a)\ell_0(a,\epsilon)
+
\bigl(1-\rho_{t-1}(a)\bigr)
\hat\ell_{t,\mathrm{emp}}^i(a,\epsilon),
\label{eq:shrunk-lat}
\end{align}
with $\hat d_t^i(a)=\hat\ell_t^i(a,\delta_t^i)$.
When the history is empty, we use the baselines directly.
As the retained history grows, the empirical estimates
receive greater weight. The verification-cost estimate
$\hat{\widetilde m}_t^i$ is initialized and shrunk
analogously using observations from completed
verifications.

\subsection{Additional experimental Details}
\label{app:extra-experiments}
For the main experiments in Section~\ref{sec:experiments}, we use the executor pool including Llama-3.2-3B-Instruct, Gemma-3-4B-IT, GPT-OSS-20B, Mistral-Small-3.2-24B-Instruct, Qwen3-32B, Llama-3.3-70B-Instruct, GPT-OSS-120B, and Mistral-Large-2512 from OpenRouter platform. 
The API providers introduce additional randomness of the output quality, latency, and cost into our experiments through fluctuations in KV-cache hit rates, service congestion from concurrent users, network conditions, etc. Network disruptions or other service failures may cause LLM API calls to fail. When a call fails, we treat its output as empty and continue executing the task graph.

\textbf{The selection of the reference model.}
We take Qwen3-32B as the reference model as it sits in the middle of the APIs pool - it is neither a cheap weak model nor a flagship but expensive one. 
The constraints of JOVE are set from unconstrained Qwen3-32B direct calls on Bamboogle, a relatively short multi-hop set with lower cost and latency than MMLU-Pro or GPQA.
Direct Qwen3-32B on Bamboogle costs \(2.00\times 10^{-4}\)~USD and takes \(14.09\)~s on average (Table~\ref{tab:main-results}); we therefore use \(  \gamma = 2.0\times 10^{-4}\ \text{USD/query}, \mu_t = 16\text{~s}.\)
The \(16\)~s deadline is the Bamboogle direct latency with a small slack. Calibrating on this easier set makes the same \((\gamma,\mu_t)\) \emph{tighter} than an average of direct Qwen3-32B cost and latency over all four datasets, where MMLU-Pro and GPQA are several times more expensive and slower.

Although JOVE may assign larger models from the pool, each such call is charged in full: under \(\gamma\) set to midsize spend on a cheap dataset, a 70B/120B assignment on one task node must be offset by cheaper assignments on other nodes and queries.
Therefore, matching accuracy at this operating point is not a free upgrade to stronger LLMs.

We do not pass an explicit thinking/reasoning flag for Qwen3-32B. OpenRouter lists thinking as optional for this model. Downstream providers (DeepInfra, SiliconFlow) may turn thinking on for some queries depending on the difficulty levels.
That behavior, when it occurs, applies to both the 32B baselines and JOVE.
% For LiveBench-Reasoning and GPQA, some models do not consistently follow the required answer format, particularly on complex queries. When the automatic evaluator cannot parse a response, we use Qwen3.8 Max to assess it against the ground-truth answer for the experimental evaluation. 

\subsubsection{Discussion on the learning sequence}
\label{app:sequence}
\begin{table}[h]
\centering
\caption{
Results on MMLU-Pro. Each query generates one task graph in the learning sequence.
}
\label{tab:mmlu}
\small
\setlength{\tabcolsep}{4pt}
\renewcommand{\arraystretch}{1.12}
\begin{tabular}{@{}lccc@{}}
\toprule
Methods
& Accuracy (\%) $\uparrow$
& Latency (sec.) $\downarrow$
& Cost ($\times 10^{-4}$ USD) $\downarrow$ \\
\midrule

\multicolumn{4}{@{}l}{
 \textit{Inference-level baselines (unconstrained)}
} \\[2pt]

\rowcolor{lightgray}
Direct (Qwen3-32B)
& \accstd{47.3}{0.8}
& \meanstd{45.39}{2.29}
& \meanstd{5.34}{0.38} \\

CoT~\citep{wei2022cot}
& \accstd{54.1}{1.1}
& \meanstd{34.11}{4.63}
& \meanstd{5.38}{0.34} \\

SoT~\citep{ning2024sot}
& \accstd{52.0}{3.1}
& \meanstd{70.17}{6.27}
& \meanstd{27.53}{3.71} \\

Plato~\citep{jin2025plato}
& \accstd{52.8}{3.2}
& \meanstd{104.20}{5.21}
& \meanstd{22.23}{2.02} \\

\midrule
\multicolumn{4}{@{}l}{
 \textit{Resource-aware methods (constrained)}
} \\[2pt]

WR-Online~\citep{li2026wiserouter}& \accstd{43.7}{4.1}& \meanstd{22.16}{2.33} & \meanstd{2.04}{0.04} \\
Knapsack-ascend & \accstd{39.1}{3.6}& \meanstd{5.88}{0.82}& \meanstd{1.70}{0.18} \\
Knapsack-descend& \accstd{42.3}{3.8}& \meanstd{7.59}{1.26}& \meanstd{1.74}{0.10} \\
\rowcolor{lightgray} JOVE single-node& \accstd{47.1}{13.1} & \meanstd{2.88}{0.71}& \meanstd{1.75}{0.13} \\
\rowcolor{highlight} \textbf{JOVE} & \accstd{\textbf{55.8}}{0.9} & \meanstd{7.32}{0.86}& \meanstd{2.13}{0.32} \\ \hline

\bottomrule
\end{tabular}%
\end{table}

As described in Section~\ref{sec:experiments}, we generate four task graphs per query, with 2, 3, 4, and 5 nodes. This yields 500 query–graph instances for Bamboogle, 48,000 for MMLU-Pro, 800 for LiveBench-Reasoning, and 1,792 for GPQA. We randomly shuffle these instances to form the online sequence. Varying the graph size exposes JOVE to different task structures and provides more decision rounds for learning, particularly on smaller datasets like Bamboogle (only 125 queries). 
Different graph sizes can change how a query is decomposed, for example, a larger graph may contain finer-grained subtasks, yielding different task graphs (different nodes) for the same query. 
We notice that for simple queries (e.g., two or three nodes are enough to clearly decompose the problem), similar nodes may appear across the four generated task graphs. 
However, this does not create direct answer leakage: JOVE neither stores nor reuses executor outputs or ground-truth answers, and every graph is executed independently. Verification labels only update the shared task–model quality estimator $\hat{\theta}_t$. 
Thus, any benefit arises through feature-based generalization across related subtasks, which is the intended learning mechanism rather than memorization of previous answers.

Table~\ref{tab:mmlu} evaluates the learning sequence in which each MMLU-Pro query appears only once, with the planner choosing its graph size. The inference-level baselines retain their results from Table~\ref{tab:main-results} because they process each query directly. JOVE achieves 55.8\% accuracy, compared with 54.8\% in the four-graph setting, with similar cost and latency. 
JOVE’s slightly higher accuracy may reflect the planner’s freedom to choose graph size. An appropriately sized graph can avoid overly broad subtasks that require expensive executors, as well as unnecessary nodes that consume budget and latency.
The other resource-aware methods likewise show no clear accuracy decline relative to their reported variability. 
The result suggests that JOVE's performance on MMLU-Pro does not depend on encountering multiple graphs for the same query. We use four graph sizes in the main experiments to provide more learning rounds, particularly for smaller datasets; the one-graph experiment provides a check under a sequence of distinct queries.

\subsubsection{Satisfaction of the constraints}
\label{app:constraint-results}
\textbf{Deadline hit rate.}
JOVE has the latency constraint:
\[
\Pr_{\pi}\!\left(L_t\leq\mu_t-d_t^{\mathrm{plan}}\,\middle|\,\mathcal{H}_t,\mathbf{a}_t\right)\geq1-\delta_t
\]
which indicates the end-to-end latency (from the start of task graph generation, to the end of the task graph execution) should be less than the target deadline $\mu_t$ with probability at least $1-\delta_t$. 

For the main experiment (Table~\ref{tab:main-results}), we set $\mu_t=16$s and $\delta_t=0.1$ for all task graphs.
Notice that JOVE supports query-specific deadlines, which users can set according to their latency needs.
We use a common deadline across queries to provide a consistent latency target for evaluation.
% We use a common deadline across queries in the experiments to ensure feasibility for all queries. 
% Setting the same $\mu_t$ across all queries does not impact the policy learning, as JOVE directly learns the distribution of each model's output latency. 

Table~\ref{tab:latency} reports the average deadline hit rate over successive rounds of online learning. 
In the final interval, every dataset reaches at least 96.0\%, exceeding the 90\% hit rate target.
MMLU-Pro, LiveBench-Reasoning, and GPQA initially fall below the 90\% target, potentially because limited call histories yield inaccurate node-level quantile estimates.
With more queries executed, each LLM has more latency samples to construct its context-based latency quantiles (detailed in Section~\ref{app:estimators}), providing precise latency estimations for the per-query MILP to decide the executor allocations. 
We notice that across all intervals, LiveBench-Reasoning remains the lowest deadline hit rate among all datasets. This can be explained by the fact that LiveBench-Reasoning queries are longer, which can increase LLM API processing time and require models to reason more, yielding high variance of the execution time.

\begin{table}[h]
\centering
\caption{Deadline hit rate (\%) across successive intervals. The target is 90\%.}
\label{tab:latency}
\begin{tabular}{@{}lllll@{}}
\toprule
Datasets            & $t\in[1,50]$ & $t\in[51,100]$ & $t\in[101,200]$ & $t\in[201,\text{end}]$ \\ \midrule
Bamboogle & $96.8 \pm 4.0$ & $98.2 \pm 1.6$ & $98.1 \pm 2.1$ & $98.6 \pm 1.4$                                  \\
MMLU-Pro & $75.6 \pm 8.8$ & $93.4 \pm 3.6$ & $95.8 \pm 1.7$ & $98.4 \pm 0.4$                       \\
LiveBench-Reasoning & $56.9 \pm 10.0$ & $89.7 \pm 4.3$ & $94.7 \pm 4.0$ & $96.0 \pm 2.1$      \\
GPQA & $72.9 \pm 12.2$ & $91.6 \pm 5.6$ & $95.7 \pm 1.3$ & $98.7 \pm 1.3$    \\ \bottomrule
\end{tabular}
\end{table}

\begin{table}[h] \centering \caption{Average cost ($\times 10^{-4}$USD) across successive intervals. $\gamma=2\times 10^{-4}$ USD.} \label{tab:cost} \begin{tabular}{@{}lcccc@{}} 
\toprule Datasets & $t\in[1,50]$ & $t\in[51,100]$ & $t\in[101,200]$ & $t\in[201,\text{end}]$ \\ 
\midrule 
Bamboogle & $11.29\pm 3.3$ & $1.83\pm 0.3$ & $1.38\small{\pm 0.2}$ & $0.98\small{\pm 0.2}$ \\ 
MMLU-Pro & $23.09\pm12.4$ & $7.65\pm4.1$ & $3.24\pm1.2$ & $1.99\pm0.1$ \\ 
LiveBench-Reasoning & $6.15\pm2.0$ & $2.39\pm1.3$ & $1.97\pm0.5$ & $1.97\pm0.4$ \\
GPQA & $21.35\pm4.6$ & $8.74\pm3.2$ & $3.05\pm1.0$ & $1.66\pm0.1$ \\ 
\bottomrule 
\end{tabular} \end{table}

\textbf{Average cost results.}
The long-term cost constraint is defined as
\[
\limsup_{T\to\infty}\frac1T\sum_{t=1}^T
 \mathbb E_\pi[C_t(\mathbf{a}_t,\mathbf{h}_t)]\leq\gamma.
\]
JOVE uses a dynamic Lagrange multiplier $\lambda_t$ to respond to budget overspending. When cost exceeds the target, $\lambda_t$ rises and makes expensive LLMs less attractive in subsequent decisions. The overall average costs in Table~\ref{tab:main-results} slightly exceed $\gamma=2\times10^{-4}$ USD per query because JOVE spends more in early rounds, when $\lambda_t$ is still small. Table~\ref{tab:cost} shows how this behavior changes over time: interval average costs decline as JOVE adjusts its selections. 
These finite-horizon results illustrate JOVE’s budget-control behavior, complementing the asymptotic budget guarantee established for the oracle policy in Proposition~\ref{prop:dual-budget}.

In early rounds, sparse call histories can cause JOVE to underestimate costs or high-percentile latencies. Cost can also exceed the target while the multiplier, initialized at $\lambda_1=0$, adjusts to observed spending. As more calls are observed, the estimates improve: in our experiments, interval-average costs approach the budget, and aggregate deadline hit rates exceed the 90\% target in the final interval. These results provide empirical support for JOVE’s deadline control as its estimates improve; establishing a conditional guarantee for every query under estimated quantiles remains future work.

The paper's focus is the reduction of the graph-level latency requirement to linear MILP constraints and the joint allocation of execution and paid verification. The resource estimators are modular and can be replaced by offline-trained or more advanced predictors. For example, offline profiling could provide better estimates at the start of learning.

\paragraph{Feasibility in the analysis and implementation.}
Our analysis characterizes resource feasibility and quality learning through complementary oracle and online settings.
For the oracle budget argument in Proposition~\ref{prop:dual-budget}, we require a stationary policy that selects allocations in $\mathcal F_t(\mathbf d_t)$ while keeping expected total cost, including planning, execution, and verification, at most $\gamma$.
This condition places the comparison policy within the same conservative node-quantile constraints used by the oracle MILP, ensuring that it is available to the optimizer.
For online quality learning, Assumption~\ref{ass:regret} requires $\mathcal F_t(\hat{\mathbf d}_t)$ to be nonempty at each round and the online MILP to be solved to optimality.
Giving the comparator the same resource estimates and budget prices as JOVE allows us to isolate the contribution of quality learning.
We acknowledge that feasibility under estimated resource coefficients alone does not establish the true latency chance constraint or long-term budget feasibility; extending the oracle resource guarantees to the implemented policy requires additional control of estimation error.
In particular, the latency argument in Appendix~\ref{app:latency-proof} applies whenever the online MILP is feasible and its estimated node durations upper-bound the corresponding true conditional latency quantiles at the assigned risk levels.
We complement these analytical results with empirical measurements of budget use and deadline hit rates, providing evidence of resource control in the evaluated workloads and motivating further analysis of online resource estimation.

\subsubsection{Additional ablations}
\label{app:extra-ablation}

\textbf{Ablation on cost and latency constraints.}
Figure~\ref{fig:constraint-ablation} examines how the resource constraints affect JOVE's final-answer accuracy. As shown in Figure~\ref{fig:budget-ablation}, increasing the cost budget generally improves accuracy because JOVE can select stronger, more expensive LLMs and invoke verification more frequently. 
Similarly, Figure~\ref{fig:latency-ablation} shows that a more relaxed latency deadline expands the set of feasible allocations, allowing JOVE to use slower but more capable LLMs. Under tighter deadlines, it instead favors faster executors, resulting in a modest reduction in accuracy. Overall, these results demonstrate that the cost budget and latency deadline provide direct and interpretable controls over the accuracy--resource trade-off.

\begin{figure}[H] \centering \begin{subfigure}{0.45\textwidth} \centering \includegraphics[width=\linewidth]{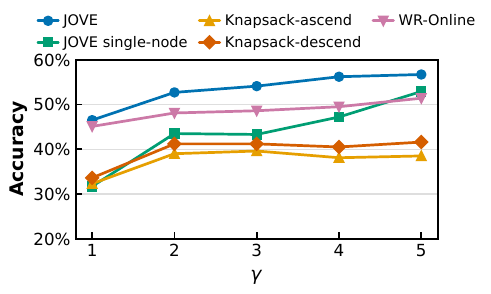} \caption{Final-answer accuracy under different API-cost budgets $\gamma$ ($\times 10^{-4}$USD).} \label{fig:budget-ablation} \end{subfigure} \hfill \begin{subfigure}{0.45\textwidth} \centering \includegraphics[width=\linewidth]{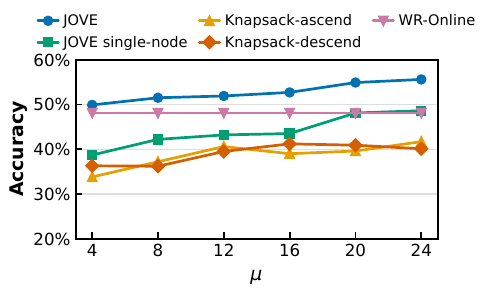} \caption{Final-answer accuracy under different latency deadlines (sec.).} \label{fig:latency-ablation} \end{subfigure} \caption{Ablation of the cost and latency constraints on GPQA.} \label{fig:constraint-ablation} 
\end{figure}

\subsubsection{Comparison with random verification}
\label{app:verification}
To isolate the benefit of the proposed verification strategy (Remark~\ref{remark:verification-rule}), we compare JOVE with a random-verification baseline. This baseline uses the same per-query MILP for executor allocation but removes the information-gain bonus from the objective and selects each node for verification with probability $p$. Both methods use the same cost budget and latency constraint.

\begin{table}[h]
\centering
\caption{Comparison of verification strategies on Bamboogle.}
\label{tab:verify}
\begin{tabular}{@{}cccc@{}}
\toprule
Methods             & Accuracy(\%)                  & Latency (s)                    & Cost ($\times 10^{-4}$USD)                       \\ \midrule
$p=0.2$ & $33.6 \pm 1.7$ & $4.23 \pm 0.83$ & $2.06 \pm 0.11$  \\
$p=0.5$ & $34.8 \pm 1.2$ & $3.27 \pm 0.35$ & $2.42 \pm 0.24$  \\
$p=0.8$ & $42.3 \pm 0.9$ & $3.09 \pm 0.72$ & $2.35 \pm 0.14$  \\
$p=0.9$ & $37.1 \pm 1.4$ & $3.02 \pm 0.40$ & $2.37 \pm 0.12$  \\
$p=1.0$ & $38.5 \pm 1.2$ & $4.00 \pm 0.39$ & $2.09 \pm 0.09$  \\
JOVE    & $49.3 \pm 1.0$ & $3.68 \pm 0.48$ & $2.18 \pm 0.18$ \\ \bottomrule
\end{tabular}
\end{table}

Table~\ref{tab:verify} shows that JOVE outperforms random verification across different probabilities, improving accuracy by 7.0 percentage points over the best random baseline ($p=0.8$). For the random baseline, low verification probabilities provide limited feedback, which can slow quality learning; frequent verification can consume too much budget that would otherwise support more capable executors, potentially explaining the lower accuracy at $p=0.9$ and $p=1.0$.
These results support selecting verification based on its information gain and cost rather than using a fixed probability.

\paragraph{Verifier reliability and the theory--experiment gap.}
Our benchmarks provide reference answers for complete queries (final responses), but not correctness labels for the intermediate subtasks generated by the planner.
Therefore, we do not have the ground truth of the node outputs. 
We acknowledge that the conditionally unbiased feedback assumed in Assumption~\ref{ass:quality} is not established for the experimental verifier.
This assumption allows verification noise to average out rather than systematically shift the estimated quality, supporting the confidence bounds around true expected correctness used in the regret proof.
Figure~\ref{fig:nodes-correctness} illustrates the practical setting: even when all five verified nodes are labeled correct, final-answer accuracy is 87.3\%.
This gap is consistent with verifier noise or imperfect alignment between node-level success criteria and final-answer correctness, rather than a direct estimate of label error rates.
Nevertheless, final-answer accuracy decreases as the number of verifier-labeled incorrect nodes increases, suggesting that the feedback carries useful quality information.
Together with JOVE's competitive final-answer accuracy and its advantage over random verification in Table~\ref{tab:verify}, these results provide empirical support that learning remains effective with the available, potentially imperfect feedback in the evaluated workloads.

\subsection{Enhanced JOVE}

\subsubsection{Joint optimization of node risk allocation}
\label{app:risk-allocation}

The main paper's formulation fixes node risk shares before optimizing the allocation. We extend it to select these shares jointly with the executor and verification decisions. This allows nodes on bottleneck paths to receive larger tolerances, reducing their required latency quantiles, while nodes on branches with slack can accommodate smaller tolerances. All paths remain subject to the response deadline.

\paragraph{Discrete risk levels.}
Before allocation, choose a small grid of risks
\[
\mathcal{Q}_t=\{r_{t,1},\ldots,r_{t,K}\}\subset(0,\delta_t].
\]
These grid levels $r_{t,k}$ can be selected from breakpoints of the empirical node-level quantile records, which are piecewise constant in the risk tolerance. For each node, executor, and risk level, we have the latency quantile
\[
\ell_t^i(a,r_{t,k})=\inf\!\left\{
d\geq0:
\Pr\!\left(D_t^i(a)\leq d | \mathcal{H}_t \right)
\geq1-r_{t,k}
\right\}.
\]
In the online implementation, these coefficients are replaced by $\hat\ell_t^i(a,r_{t,k})$, using the same estimator and observation histories as the main algorithm.

\paragraph{MILP formulation.}
To optimize the allocation of risk levels across the graph, 
we replace $z_t^i(a,\eta)$ with binary variables $x_t^i(a,\eta,k)$,
where $x_t^i(a,\eta,k)=1$ selects executor $a$, verification decision
$\eta$, and risk level $r_{t,k}$ at node $i$.
The original allocation and the selected risk share are recovered as
\[
z_t^i(a,\eta)=\sum_{k=1}^{K}x_t^i(a,\eta,k),
\qquad
\delta_t^i=
\sum_{a,\eta,k}r_{t,k}x_t^i(a,\eta,k).
\]
Let $\mathbf{x}_t=\{x_t^i(a,\eta,k)\}_{a\in\mathcal{A}_t^i, \eta\in\{0,1\},k=1,\dots,K}$.
Define the allocated node duration as
\[
v_t^i(\mathbf{x}_{t})=\sum_{a\in\mathcal{A}_t^i}\sum_{\eta\in\{0,1\}}\sum_{k=1}^{K}
\ell_t^i(a,r_{t,k})
x_t^i(a,\eta,k).
\]
The extended problem is
\begin{equation}
\label{eq:risk-allocation-milp}
\begin{aligned}
\max_{\mathbf{x}_t,\mathbf{F}_t}
& 
\sum_{i\in\mathcal N_t}
\sum_{a\in\mathcal A_t^i}
\sum_{\eta\in\{0,1\}}\sum_{k=1}^K
\left[
\frac{1}{|\mathcal{N}_t|}(\mathrm{UCB}_t^i(a)
+
k_v\eta I_t^i(a))
-
 \lambda_t\hat c_t^i(a,\eta)
\right]x_t^i(a,\eta,k)
\\
\text{s.t.}\quad
& \sum_{a,\eta,k}x_t^i(a,\eta,k)=1,
\qquad \forall i\in\mathcal N_t,\\
&\sum_{i,a,\eta,k}r_{t,k}x_t^i(a,\eta,k)\le\delta_t,\\
& F_t^i\ge v_t^i(x),
\qquad \forall i\in\mathcal S_t,\\
& F_t^i\ge F_t^j+v_t^i(x),
\qquad \forall(j,i)\in\mathcal E_t,\\
& F_t^i\le\mu_t-d_t^{\mathrm{plan}},
\qquad \forall i\in\mathcal T_t,\\
& F_t^i\ge0,
\qquad \forall i\in\mathcal N_t,\\
& x_t^i(a,\eta,k)\in\{0,1\},
\qquad \forall i,a,\eta,k.
\end{aligned}
\end{equation}
All objective and constraint coefficients are fixed before solving,
so \eqref{eq:risk-allocation-milp} remains a MILP, which is solvable using the same procedure as JOVE. 
The risk grid increases the number of binary variables by a factor
of $K$ and adds one global risk-budget constraint.

This extended version of JOVE incorporates node-level latency risk allocation directly into the per-query MILP. Compared with the uniform allocation of $\delta_t^i$ used in Section~\ref{sec:method-milp}, optimizing the node-level risks inside the MILP generally yields tighter (less conservative) latency bounds, because the risk budget is concentrated on the critical path according to the per-node latency estimates. For instance, if one source–sink path contains a single node while another contains four, there is little benefit to spending risk budget on the single-node path; the budget is better spent loosening the bounds of the nodes that actually determine end-to-end latency.

% \paragraph{Empirical comparison.}
% We compare uniform and jointly optimized risk allocation using
% the same resource estimates, objective, and budget and latency
% requirements. The results show 

\newpage

\subsection{Notation}
\label{app:notation}
 
Tables~\ref{tab:notation} and~\ref{tab:notation-learning} collect the symbols used throughout the paper, grouped by
role: the query and its task graph, the allocation decisions, the stochastic
outcomes of execution and verification, the oracle statistics of
Section~\ref{sec:method-milp}, their online estimates from
Section~\ref{sec:method-estimation}, the resource constraints and control
parameters, and the constants appearing in the analysis of
Section~\ref{sec:method-theory}. Superscript $i$ indexes a node, subscript $t$
indexes a query, and $a$ denotes a candidate executor. A bar ($\bar\cdot$)
marks an expected or oracle quantity and a hat ($\hat\cdot$) marks its online
estimate.
 
\begin{table}[ht]
\centering
\caption{Notation for the task graph, decisions, outcomes, and oracle statistics.}
\label{tab:notation}
\renewcommand{\arraystretch}{1.05}
\resizebox{0.95\linewidth}{!}{%
\begin{tabular}{@{}p{0.28\textwidth}p{0.70\textwidth}@{}}
\toprule
\textbf{Symbol} & \textbf{Description} \\
\midrule
\multicolumn{2}{@{}l}{\textit{Query and task graph}}\\
$t$, $T$ & Query index and horizon (number of queries). \\
$x_t$ & The $t$-th query. \\
$\mathcal G_t=(\mathcal N_t,\mathcal E_t)$ & Planner-generated task graph (DAG) for query $x_t$. \\
$\mathcal N_t$, $|\mathcal N_t|$ & Node (subtask) set and its cardinality. \\
$\mathcal E_t$ & Edge set; $(j,i)\in\mathcal E_t$ means node $i$ requires the output of node $j$. \\
$\mathcal S_t$, $\mathcal T_t$ & Source and sink nodes of $\mathcal G_t$. \\
$\mathcal P(\mathcal G_t)$ & Set of source-to-sink paths in $\mathcal G_t$. \\
$\mathcal A_t^i$ & Candidate executor set at node $i$, chosen by the planner. \\
$\pi$ & Online allocation policy. \\
$\mathcal H_t$ & Information available before allocating query $x_t$.\\
\addlinespace
 
\multicolumn{2}{@{}l}{\textit{Allocation decisions}}\\
$a_t^i$, $\mathbf a_t$ & Executor selected at node $i$; the vector of executors for query $t$. \\
$h_t^i\in\{0,1\}$, $\mathbf h_t$ & Verification indicator at node $i$; the vector of verification decisions. \\
$\eta\in\{0,1\}$ & Verification value in the joint encoding. \\
$z_t^i(a,\eta)$, $\mathbf z_t$ & Binary indicator that node $i$ uses LLM $a$ with verification decision $\eta$. \\
$F_t^i$ & Finish time of node $i$, relative to the start of graph execution. \\
$v_t^i(\mathbf z_t^i)$ & Latency quantile selected at node $i$ by the allocation. \\
$\mathcal F_t(\mathbf d_t)$ & Feasible set of assignment and finish-time variables induced by node durations $\mathbf d_t$. \\
\addlinespace
 
\multicolumn{2}{@{}l}{\textit{Stochastic outcomes (uppercase) and their realizations (lowercase)}}\\
$Y_t^i(a)\in\{0,1\}$ & Latent correctness of node $i$ executed by model $a$; not observed from execution. \\
$\widehat Y_t^i\in\{0,1\}$ & Binary label returned by the fixed verifier when $h_t^i=1$. \\
$M_t^i(a)$, $m_t^i(a)$ & Stochastic and realized API cost of executing node $i$ with model $a$. \\
$\widetilde M_t^i$, $\widetilde m_t^i$ & Stochastic and realized verification cost at node $i$. \\
$D_t^i(a)$, $d_t^i(a)$ & Stochastic and realized execution latency of node $i$ under model $a$. \\
$M_t^{\mathrm{plan}}$, $m_t^{\mathrm{plan}}$ & Stochastic and realized API cost of task-graph generation ($0$ for a local planner). \\
$d_t^{\mathrm{plan}}$ & Graph-generation plus the time reserved for allocation. \\
$L_t$ & Graph execution latency, equal to the critical-path duration. \\
$C_t(\mathbf a_t,\mathbf h_t)$, $c_t$ & Total and realized API cost of processing query $x_t$ (planning, execution, verification). \\
\addlinespace
 
\multicolumn{2}{@{}l}{\textit{Oracle statistics (Section~\ref{sec:method-milp})}}\\
$\bar y_t^i(a)$ & Expected node correctness, $\mathbb E[Y_t^i(a)]$. \\
$\bar c_t^i(a,\eta)$ & Expected execution plus verification cost at node $i$. \\
$\ell_t^i(a,\epsilon)$ & Smallest latency met with probability at least $1-\epsilon$. \\
$\bar d_t^i(a)=\ell_t^i(a,\delta_t^i)$, $\mathbf d_t$ & Node-level latency quantile at risk share $\delta_t^i$; the collection over all $(i,a)$. \\
\bottomrule
\end{tabular}%
}
\end{table}

\begin{table}[ht]
\centering
\caption{Notation for online learning, resource control, and the analysis.}
\label{tab:notation-learning}
\renewcommand{\arraystretch}{1.05}
\resizebox{0.95\linewidth}{!}{%
\begin{tabular}{@{}p{0.28\textwidth}p{0.70\textwidth}@{}}
\toprule
\textbf{Symbol} & \textbf{Description} \\
\midrule
\multicolumn{2}{@{}l}{\textit{Online estimation (Section~\ref{sec:method-estimation})}}\\
$\phi_t^i(a)\in\mathbb R^d$, $d$ & Feature vector of the task--model pair and its dimension. \\
$\mathcal D_{t-1}$ & Feature--label pairs obtained through verification before query $t$. \\
$\Upsilon_{t-1}(a)$ & Execution records $(\phi,m,d)$ for model $a$ collected before query $t$. \\
$w(\phi,\phi_t^i(a))$ & Softmax similarity weight assigned to a past record. \\
$G_{t-1}$, $G_{t-1}^{+}$ & Ridge information matrix before query $t$, and after adding one verified observation. \\
$\hat\theta_t$ & Ridge-regression estimate of the quality parameter. \\
$u_t^i(a)$ & Uncertainty score, $\lVert\phi_t^i(a)\rVert_{G_{t-1}^{-1}}$. \\
$\mathrm{UCB}_t^i(a)$ & Optimistic quality estimate, clipped to $[0,1]$. \\
$I_t^i(a)$ & D-optimal information gain from verifying $(i,a)$. \\
$\hat m_t^i(a)$, $\hat{\widetilde m}_t^i$ & Estimated execution and verification cost. \\
$\hat\ell_t^i(a,\epsilon)$, $\hat d_t^i(a)$, $\hat{\mathbf d}_t$ & Estimated latency quantile, node duration at risk $\delta_t^i$, and their collection. \\
$\hat c_t^i(a,\eta)$ & Estimated execution plus verification cost, $\hat m_t^i(a)+\eta\hat{\widetilde m}_t^i$. \\
$n_{t-1}(a)$, $\rho_{t-1}(a)$ & Size of the retained history for model $a$ and the corresponding baseline weight. \\
$m_0(a)$, $\ell_0(a,\epsilon)$ & Model-specific baseline cost and latency quantile used before observations accumulate. \\
\addlinespace
 
\multicolumn{2}{@{}l}{\textit{Resource constraints and control parameters}}\\
$\gamma$ & Long-run average API-cost budget per query. \\
$\mu_t$ & Response deadline for query $t$. \\
$\delta_t$, $\delta_t^i$ & Latency violation tolerance for query $t$ and the risk share assigned to node $i$. \\
$\lambda_t$, $\alpha_\lambda$ & Budget shadow price and its step size. \\
$\lambda_{\mathrm{reg}}$ & Ridge regularization parameter. \\
$\beta_t$ & Confidence radius scaling the uncertainty bonus. \\
$k_v\geq0$ & Verification weight on the information-gain bonus. \\
\addlinespace
 
\multicolumn{2}{@{}l}{\textit{Analysis (Section~\ref{sec:method-theory} and Appendix~\ref{app:regret-proof})}}\\
$\theta^\star$ & True quality parameter, $\bar y_t^i(a)=\phi_t^i(a)^\top\theta^\star$. \\
$L$, $S$ & Bounds on $\lVert\phi_t^i(a)\rVert_2$ and $\lVert\theta^\star\rVert_2$. \\
$\varepsilon_t^i$, $R$ & Verification label noise and its sub-Gaussian parameter. \\
$\delta_{\mathrm{conf}}$ & Failure probability of the confidence event. \\
$n_{\max}$, $\lambda_{\max}$, $\widetilde m_{\max}$ & Bounds on nodes per query, budget price, and estimated verification cost. \\
$N=Tn_{\max}$ & Upper bound on the total number of executed nodes over the horizon. \\
$c_v=\lambda_{\max}\widetilde m_{\max}$ & Upper bound on the priced verification cost. \\
$C_{\mathrm{bat}}$, $\Gamma_T$ & Batching constant and cumulative-uncertainty bound. \\
$\tau_T$, $K_T$ & Uncertainty threshold and lower endpoint of the admissible verification-weight interval. \\
$\mathcal V_t$, $b_t$ & Nodes selected for verification at query $t$ and their number. \\
$J_t(\mathbf z_t)$, $\mathbf z_t^\star$ & Comparator objective with true qualities, and its maximizer. \\
$\mathcal R_T$ & Quality-learning regret, $\sum_{t=1}^T[J_t(\mathbf z_t^\star)-J_t(\mathbf z_t)]$. \\
\bottomrule
\end{tabular}%
}
\end{table}

\stopcontents[appendix]

\end{document}